\documentclass{article}

\usepackage{microtype}
\usepackage{graphicx}
\usepackage[dvipsnames, table]{xcolor}
\usepackage{subcaption}
\usepackage{booktabs} 

\usepackage{hyperref}

\newcommand{\pmerr}[2]{#1$_{\pm \text{#2}}$}

\usepackage[preprint]{icml2026}

\usepackage{amsmath}
\usepackage{amssymb}
\usepackage{mathtools}
\usepackage{amsthm}
\usepackage{microtype}
\usepackage{enumitem}
\usepackage{algorithm}
\usepackage{booktabs}
\usepackage{bm}

\usepackage{tikz}
\usetikzlibrary{shapes, arrows.meta, positioning, calc, backgrounds, patterns}

\definecolor{myblue}{RGB}{70,130,180}
\definecolor{myred}{RGB}{220,60,60}
\definecolor{mygrey}{RGB}{128,128,128}
\definecolor{noisearea}{RGB}{240,240,240}

\usepackage[capitalize,noabbrev]{cleveref}

\theoremstyle{plain}
\newtheorem{theorem}{Theorem}[section]
\newtheorem{proposition}[theorem]{Proposition}
\newtheorem{lemma}[theorem]{Lemma}
\newtheorem{corollary}[theorem]{Corollary}
\theoremstyle{definition}
\newtheorem{definition}[theorem]{Definition}
\newtheorem{assumption}[theorem]{Assumption}
\theoremstyle{remark}
\newtheorem{remark}[theorem]{Remark}

\newcommand{\R}{\mathbb{R}}
\newcommand{\E}{\mathbb{E}}
\newcommand{\Prob}{\mathbb{P}}
\newcommand{\1}{\mathbf{1}}

\newcommand{\TV}{\mathrm{TV}}

\newcommand{\Haar}{\lambda} 
\newcommand{\G}{\mathcal{G}}
\newcommand{\X}{\mathcal{X}}
\newcommand{\M}{\mathcal{M}}
\newcommand{\Z}{\mathbf{Z}}
\newcommand{\Hc}{\mathbf{H}}
\newcommand{\A}{\mathbf{A}}
\newcommand{\W}{\mathbf{W}}
\newcommand{\D}{\mathbf{D}}
\newcommand{\Xc}{\mathbf{X}}
\newcommand{\grad}{\nabla}
\newcommand{\norm}[1]{\left\lVert #1\right\rVert}

\newcommand{\push}{\#} 
\newcommand{\acts}{\cdot} 
\newcommand{\Orb}{\mathcal{O}}

\crefname{table}{table}{Table}
\crefname{figure}{figure}{Figure}

\definecolor{softpink}{HTML}{FFB6C1} 
\definecolor{softpink_lighter}{HTML}{FFDFE5} 

\usepackage[textsize=tiny]{todonotes}

\icmltitlerunning{Rethinking Diffusion Models with Symmetries through Canonicalization}

\begin{document}

\twocolumn[
  \icmltitle{Rethinking Diffusion Models with Symmetries through Canonicalization\\ with Applications to Molecular Graph Generation}



  \icmlsetsymbol{equal}{*}

  \begin{icmlauthorlist}
    \icmlauthor{Cai Zhou}{equal,mit}
    \icmlauthor{Zijie Chen}{equal,zju}
    \icmlauthor{Zian Li}{pku}
    \icmlauthor{Jike Wang}{zju}
    \icmlauthor{Kaiyi Jiang}{mit,pt}\\
    \icmlauthor{Pan Li}{gt}
    \icmlauthor{Rose Yu}{ucsd}
    \icmlauthor{Muhan Zhang}{pku}
    \icmlauthor{Stephen Bates}{mit}
    \icmlauthor{Tommi Jaakkola}{mit}
  \end{icmlauthorlist}

  \icmlaffiliation{mit}{Massachusetts Institute of Technology}
  \icmlaffiliation{zju}{Zhejiang University}
  \icmlaffiliation{pku}{Peking University}
  \icmlaffiliation{gt}{Georgia Institute of Technology}
  \icmlaffiliation{pt}{Princeton University}
  \icmlaffiliation{ucsd}{University of California, San Diego}

  \icmlcorrespondingauthor{Cai Zhou}{caiz428@mit.edu}

  \icmlkeywords{Canonicalization, Equivariance, Diffusion Models, Molecular Graph Generation}
  \vskip 0.3in
]



\printAffiliationsAndNotice{\icmlEqualContribution}

\begin{abstract}
  Many generative tasks in chemistry and science involve distributions invariant to group symmetries (e.g., permutation and rotation). A common strategy enforces invariance and equivariance through architectural constraints such as equivariant denoisers and invariant priors. In this paper, we challenge this tradition through the alternative canonicalization perspective: first map each sample to an orbit representative with a canonical pose or order, train an unconstrained (non-equivariant) diffusion or flow model on the canonical slice, and finally recover the invariant distribution by sampling a random symmetry transform at generation time. Building on a formal quotient-space perspective, our work provides a comprehensive theory of canonical diffusion by proving: (i) the correctness, universality and superior expressivity of canonical generative models over invariant targets; (ii) canonicalization accelerates training by removing diffusion score complexity induced by group mixtures and reducing conditional variance in flow matching. We then show that aligned priors and optimal transport act complementarily with canonicalization and further improves training efficiency. We instantiate the framework for molecular graph generation under $S_n \times SE(3)$ symmetries. By leveraging geometric spectra-based canonicalization and mild positional encodings, canonical diffusion significantly outperforms equivariant baselines in 3D molecule generation tasks, with similar or even less computation. Moreover, with a novel architecture \emph{Canon}, CanonFlow achieves state-of-the-art performance on the challenging GEOM-DRUG dataset, and the advantage remains large in few-step generation.
\end{abstract}

\section{Introduction}

Generative modeling is fundamentally grounded in the geometric structure of data. In domains such as computer vision and natural language processing (NLP), data exhibits specific symmetries—for instance, objects in images possess translation invariance, while semantic meaning in text relies on sequential order. 
Recent diffusion~\citep{GenerativeGradients,ho2020denoising,song2020score} and flow-based~\citep{liu2022flow,lipman2022flow,albergo2023stochastic} generative models have achieved great success across images~\citep{dhariwal2021diffusion,rombach2022high}, videos~\citep{ho2022video,singer2022make}, text~\citep{li2022diffusionlm,gong2022diffuseq}, and biomolecules~\citep{watson2023rfdiffusion,corso2023diffdock}. 
Crucially, however, modalities like images are not fully invariant to all transformations: an upside-down landscape or a reversed sentence typically loses its semantic validity. Consequently, state-of-the-art diffusion models in these fields explicitly break symmetries: they utilize fixed grid topologies or inject positional encodings (PEs) to anchor the generation process to a canonical orientation.

In contrast, molecular generation operates on a fundamentally different geometric space defined by the direct product of permutation and Euclidean symmetries, $S_N \times SE(3)$. Unlike natural images, a rotated molecule or a re-indexed molecular graph represents the exact same physical entity.
Molecular generative models typically enforce the constraints by building equivariant architectures and/or invariant priors~\citep{hoogeboom2022equivariant,xu2022geodiff,tian2024equiflowequivariantconditionalflow}, so that the learned vector field (score/velocity) respects the group action and the generated distribution is symmetry-consistent. While principled, this approach often incurs substantial architectural and computational overhead (e.g., equivariant layers, tensor algebra), and it can obscure an additional challenge: symmetry creates latent “gauge” ambiguity, so intermediate noisy states may correspond to multiple equivalent group-transformed configurations. This complex mixture-like nature results in  ``trajectory crossing” and conflicting gradients, as the model struggles to determine \textit{which} valid orientation to generate from a symmetric noise vector, making the learned dynamics less straight.

\begin{figure}[t]
    \centering
    \includegraphics[width=\linewidth]{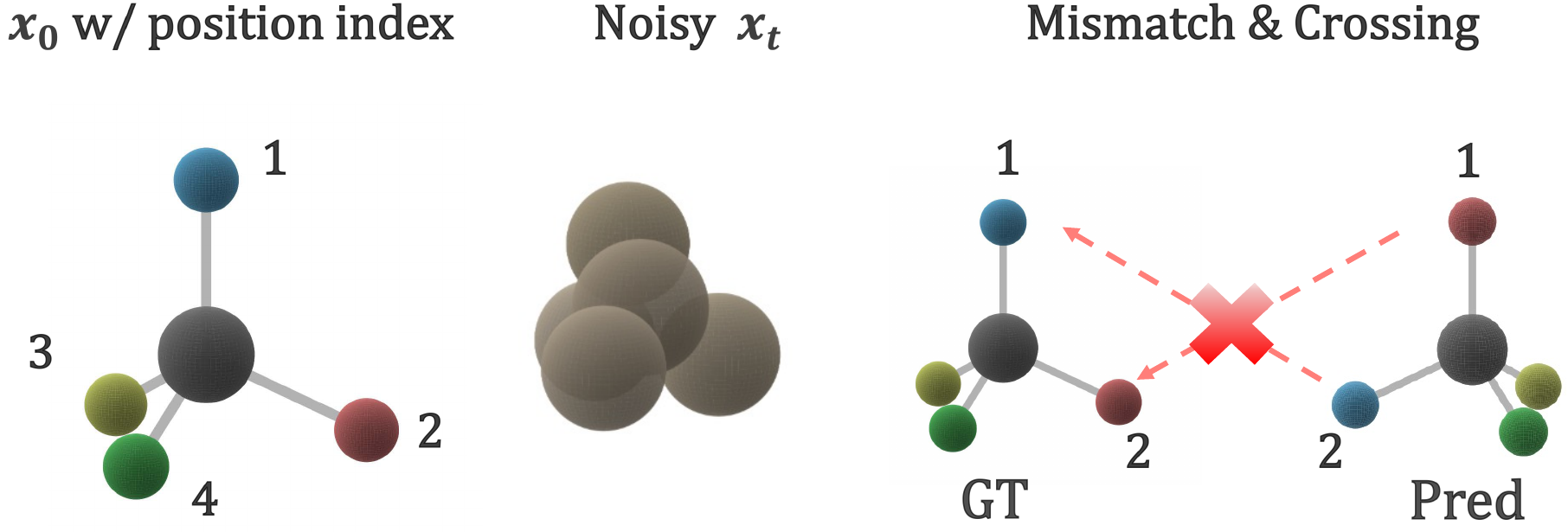}
    \caption{Motivation of Canonical Diffusion: Position-wise supervision induces mismatch loss under symmetries for equivalent molecules in diffusion models.}
    \label{fig:motivation}    
    \vspace{-0.5cm}
\end{figure}
Motivated by this observation, we argue that the efficiency of image generation stems precisely from its lack of total invariance, and we propose to transfer this advantage to molecular modeling via \textbf{canonicalization}. In this work, we challenge the necessity and effectiveness of equivariant generative models, and introduce a novel canonical diffusion framework, with theoretical analysis of how this symmetry breaking procedure accelerates diffusion training while ensuring validity over invariant or equivariant targets. 
By mapping the quotient space of molecules to a unique canonical section—explicitly breaking symmetry during training—we align the noise and data distributions. This effectively resolves the trajectory mismatch problem, transforming the complex equivariant generation task into a simplified transport problem along a canonical manifold. 

Our starting point is the flow-matching identity that the irreducible regression error equals an expected conditional variance of endpoint displacements given an intermediate state. In \Cref{sec:theory}, under a group-aligned lifting construction, we show this conditional variance decomposes into two nonnegative components: a within-slice term that reflects genuine transport difficulty on a canonical slice, and a symmetry-ambiguity term that arises purely from marginalizing over the latent group element. Canonicalization eliminates the symmetry-ambiguity term by fixing a gauge (working on a canonical representative per orbit), whereas optimal-transport-like couplings target the within-slice term by making the coupling more Monge-like and the trajectory closer to straight-line transport. 

Based on this analysis, in \Cref{sec:method} we propose a practical symmetry-aware pipeline for molecular generation that combines a canonicalizer with flexible non-equivariant backbones.
We canonicalize molecules by jointly fixing permutation and rotation gauges (e.g., a Fiedler-vector rank for $S_N$ and a rank-anchored frame for $SO(3)$), then train diffusion/flow models directly in canonical space with a gap-free prior and canonical conditions. 
This yields a simple yet powerful alternative to fully equivariant generative modeling: it reduces symmetry-induced ambiguity, improves few-step accuracy with the straighter learned transport, and unlocks the expressive and computational advantages of generic Transformers/GNNs for large-scale molecular generation. Furthermore, we develop a novel architecture termed \emph{\textbf{Canon}}~(\Cref{fig:canon-arch}) which explicitly incorporates and refines canonical information in an additional atom hidden state, enabling canonicality-aware denoising.

Experimentally, canonicalization is compatible with any (equivariant or non-equivariant) base models. \Cref{sec:exp} shows that the canonicalized counterparts of baselines such as SemlaFlow~\citep{irwin2024semlaflow} yields significantly better results on popular unconditional 3D molecule generation benchmarks including QM9 and GEOM-DRUG, with completely ignorable computation overheads. Furthermore, by leveraging the ``Canon" architecture, \emph{CanonFlow} leads to state-of-the-art molecule stability and validity metrics, outperforming all baselines with large margins. Remarkably, canonicalized models reveal faster training convergence and better few-step sampling quality.

\section{Preliminary}\label{sec:preliminary}

\paragraph{Canonicalization in symmetric data.}

We work on a measurable space $\M$ of structured objects. In molecular generation, we use the following state throughout the paper:
\[
\Z = (\Xc,\Hc, \A)\in \M,
\]
where $\Xc\in \R^{N\times 3}$ are coordinates; $\Hc\in \R^{N\times d_h}$ are atom features with each element consists of a combination of $d_h$ scalars, for example, atom types $\{1,\dots,K\}^N$ and formal charges; $\A\in\R^{N\times N\times d_e}$ encodes bonds/edge features.

Let a group action $\G$ act on $\M$ by
\begin{equation}
\begin{aligned}
g\acts \Z \;=\; (R,\;t,\;\pi)\acts (\Xc,\Hc,\A)\;:=\;\\(\pi(\Xc)R^\top + \1 t^\top,\;\pi(\Hc),\;\pi(\A)),
\end{aligned}
\end{equation}
where $(R,t)\in SE(3)$ and $\pi\in S_N$ (simultaneous permutation of nodes and adjacency).

\begin{definition}[Orbit and quotient]
The orbit of $\Z$ is $\Orb(\Z)=\{g\acts \Z: g\in\G\}$. The quotient space is $\M/\G=\{\Orb(\Z):\Z\in\M\}$.
\end{definition}

\begin{definition}[Invariance]
A probability measure $\mu$ on $\M$ is $\G$-invariant if $g\push\mu=\mu$ for all $g\in\G$.
\end{definition}
We now introduce \emph{canonicalization} as orbit selection.
\begin{definition}[Canonicalization map and slice]\label{def:map_slice}
A measurable map $\Psi:\M\to \M$ is a \emph{canonicalization map} if (i) $\Psi(\Z)\in\Orb(\Z)$ and (ii) $\Psi(g\acts \Z)=\Psi(\Z)$ for all $g\in\G$. The image $S:=\Psi(\M)$ is a \emph{canonical slice}.
\end{definition}

\begin{remark}[Stabilizers and non-uniqueness]
If $\Z$ has non-trivial stabilizer $\mathrm{Stab}(\Z)=\{g:g\acts \Z=\Z\}$ (e.g., graph automorphisms or symmetric geometries), then canonical representatives may be non-unique. This inherently causes discontinuities or multi-valued choices; weighted/probabilistic frames are one remedy \cite{dym2024equivariant}.
\end{remark}
Fortunately, we have the following assumption that is often reasonable for noisy real-world 3D molecules (exact symmetries are measure-zero).

\begin{assumption}[Free action a.s.]\label{ass:free}
Under $p_0$, the stabilizer is trivial almost surely.
\end{assumption}

There exists a finite, translation-invariant measure (Haar measure $\lambda$) that can be normalized to have total mass 1 on a compact topological group.
For permutations $S_N$, Haar is uniform; for $SO(3)$, Haar is the uniform rotation measure; for translations, however, Haar is not finite. Thus following previous work~\citep{hoogeboom2022equivariant, li2024geometric}, in practice we remove translations by centering (e.g., center-of-mass), leaving a compact symmetry $SO(3)\times S_N$.

\begin{assumption}[Centered representation]\label{ass:center}
Throughout the paper, we assume $\Xc$ is zero-centered so global translations are removed; the remaining group is compact: $\G = SO(3)\times S_N$ with Haar probability $\Haar$.
\end{assumption}

\paragraph{Diffusion and flow-matching models.}
We present a continuous-time formulation (discrete-time DDPM~\citep{ho2020denoising} and CTMC variants follow similarly).

Let $\Z_0\sim p_0$ be data. We first present score-based diffusion~\citep{song2020score}. Consider a forward noising SDE on an ambient Euclidean embedding of $\M$ (continuous parts such as $\Xc$), a standard VP-SDE is
\begin{equation}
\mathrm{d}\Z_t = -\frac{1}{2}\beta(t)\Z_t\,\mathrm{d}t + \sqrt{\beta(t)}\,\mathrm{d}W_t,
\end{equation}
with marginals $p_t$ and score $s_\star(\Z,t)=\grad_{\Z}\log p_t(\Z)$, which is learned by a parameterized score network. 

Some recent generative models are built from the equivalent flow matching viewpoint, learning a time-dependent vector field $v_t(\Z)$ that transports a prior $p_1$ to data $p_0$ via an ODE:
\begin{equation}\label{eq:ode}
\frac{\mathrm{d}\Z_t}{\mathrm{d}t}=v_t(\Z_t),\qquad \Z_1\sim p_1,\quad \Z_0\sim p_0.
\end{equation}
In conditional flow matching, one specifies a coupling $\gamma(\Z_0,\Z_1)$ and a \emph{microscopic} conditional path $\Z_t=\Phi_t(\Z_0,\Z_1)$ with conditional vector field $u_t(\cdot\mid \Z_0,\Z_1)$. The optimal marginal field is the conditional expectation:
\begin{equation}\label{eq:marginal-field}
v_t(\Z)\;=\;\E_{\gamma}\big[u_t(\Z\mid \Z_0,\Z_1)\mid \Z_t=\Z\big].
\end{equation}
The categorical parts (such as $\A,\Hc$), in comparison, are usually modeled by an appropriate discrete diffusion, typically with uniform/absorbing/masked states as the prior.

\section{Understanding Canonicalization in Diffusion Models with Symmetries}\label{sec:theory}

Building on a formal quotient-space perspective, we provide a comprehensive theory of canonical diffusion in this section. We prove (i) a factorization theorem for invariant measures via slice distributions and Haar randomization, and the universality and superior expressivity of canonicalized generative models over invariant targets~(\Cref{subsec:correct_superior}); (ii) an explicit expression for the diffusion score complexityinduced by group mixtures while being removable through canonicalization, and a flow-matching conditional variance decomposition showing that canonicalization reduces conditional variance, thereby accelerating diffusion/flow matching training~(\Cref{subsec:accelerate}). 
Complete proof and more discussions are available in Appendix~\ref{sec_appendix:proof}.

\subsection{Canonical Generative Models Induce Universal Invariant Distributions}\label{subsec:correct_superior}

This subsection formalizes why learning on the canonical slice while sampling with Haar randomization is not only sufficient in representing any invariant target~(\Cref{subsubsec:correct}), but also more expressive with non-equivariant denoising backbones~(\cref{subsubsec:superior}).

\subsubsection{Universality of Canonical Parameterizations over Invariant and Equivariant Targets}\label{subsubsec:correct}
First, we show that invariant distributions factor through the canonical slice. Recall that Harr probability $\lambda$ is always uniform for $S_N\times SO(3)$.
\begin{theorem}[Factorization of invariant measures; known]\label{thm:factor}
Suppose Assumptions~\ref{ass:free} and \ref{ass:center} hold. Let $\mu$ be any $\G$-invariant probability measure on $\M$. Let $\Psi$ be an orbit representative map defined $\mu$-a.s., and let $\nu=\Psi\push\mu$ be the slice distribution on $S=\Psi(\M)$. Then
\begin{equation}\label{eq:factor}
\mu \;=\; \int_{S}\left(\int_{\G}\delta_{g\acts \Z}\,\mathrm{d}\Haar(g)\right)\mathrm{d}\nu(\Z).
\end{equation}
Equivalently, if $\tilde\Z\sim \nu$, $g\sim\Haar$ independent, then $g\acts \tilde\Z\sim \mu$.
\end{theorem}
\begin{corollary}[Sufficiency of slice modeling]\label{cor:suff}
To model any invariant target $\mu$, it suffices to model the slice distribution $\nu$; invariance is recovered by Haar randomization.
\end{corollary}
We propose that canonicalization is a general technique for constructing equivariant/invariant functions \emph{without equivariant backbones}, generalizing \citet{kaba2023canonical}:
\begin{proposition}[Universality of canonicalized parameterzations over invariant and equivariant targets]\label{thm:universal_inv_equi}
Let $\G$ act continuously on a compact set $K\subset \R^d$. Suppose we have a (measurable) canonicalization map $\Psi:K\to K$ as \Cref{def:map_slice}, and a (measurable) gauge map $\kappa:K\to \G$ s.t.
\begin{equation}\label{eq:gauge-decomp}
\Psi(g\acts x)=\Psi(x),\ \kappa(g\acts x)=g\kappa(x),\ x=\kappa(x)\acts \Psi(x)
\end{equation}
Consider the parametrization
\begin{equation}
\phi(x)=\kappa(x)\acts f(\Psi(x)),
\end{equation}
where $f$ is a universal approximator on $\Psi(K)$. Then $\phi$ is a universal approximator of continuous $\G$-equivariant functions on $K$, and $f\circ \Psi$ is universal for continuous $\G$-invariant functions on $K$.
\end{proposition}

\subsubsection{Stronger Expressivity via Symmetry Breaking}\label{subsubsec:superior}

We now argue that symmetry breaking with the help of non-equivariant models can be practically more expressive. 
Even though the true score of an invariant distribution is equivariant, enforcing equivariance inside the denoiser can \textbf{restrict architectural choices}, \emph{limiting the expressivity while increasing non-necessary computation costs}~\citep{yan2023swingnn}. 
Remarkably, non-equivariant models have stronger expressivity of their equivariant counterparts~\citep{zhang2021labeling,klWL}. However, frame averaging or relational pooling are needed to recover invariant outputs, leading to a complexity of the order of the group.
This motivates canonicalized diffusion: it uses a symmetry-breaking gauge (canonical order/pose) and leverages a stronger non-equivariant backbone on the slice, then restores invariance by randomization. Without the necessity to traverse and average, cononicalized non-equivariant models yield stronger expressivity with similar or less computation. More details are available in Appendix~\ref{subsubsec_appendix:proof_superior}.

\subsection{Canonicalization Accelerates Diffusion Training}\label{subsec:accelerate}

In this subsection, we provide two complementary analyses, score mixture complexity~(\cref{subsubsec:score}) and flow-matching conditional variance~(\cref{subsubsec:variance}), rigorously formalizing how canonicalization decreases symmetry-induced variance and accelerates diffusion or flow model training. We further show that aligned canonical prior and optimal transport further reduces within-slice~(\Cref{subsubsec:align_ot}).

\subsubsection{Mixture Structure in Diffusion Score}\label{subsubsec:score}
For simplicity, consider a finite group $\G$ with $M$ elements acting orthogonally on $\R^d$ and $g_m\in M$ the $m$-th element. Let $q$ be a slice density supported on a canonical region $S$. Define the invariant mixture
\begin{equation}\label{eq:finite-mixture}
p(z_t)=\frac{1}{M}\sum_{m=1}^M q(g_m^{-1}\acts z_t).
\end{equation}
where $z_t$ refers to noisy states, and the summation can be substituted by integration for an infinite group.
Assume $q$ is differentiable and positive where needed. Given the chain rule and orthogonality, the score follows
\begin{equation}
\begin{aligned}
\grad \log p(z_t)&=\sum_{m=1}^M w_m(z_t)\, g_m\acts \grad \log q(g_m^{-1}\acts z_t),\\
w_m(z_t):&=\frac{q(g_m^{-1}\acts z_t)}{\sum_{j=1}^M q(g_j^{-1}\acts z_t)}.
\end{aligned}
\end{equation}
Where multiple group copies overlap, the responsibilities $w_m(x_t)$ vary rapidly. The induced sharp score fields require more capacity and smaller step sizes for stable reverse integration. As noted in \citep{yan2023swingnn}, invariant training leads to increased modes and mixture-like optimal denoising scores (GMM analogy). Instead, canonicalization removes the mixture: there is a single component on the slice ($M=1$), which smooths the score landscape, simplifying denoising score network training.

\begin{figure*}[t]
    \centering
    \includegraphics[width=\textwidth]{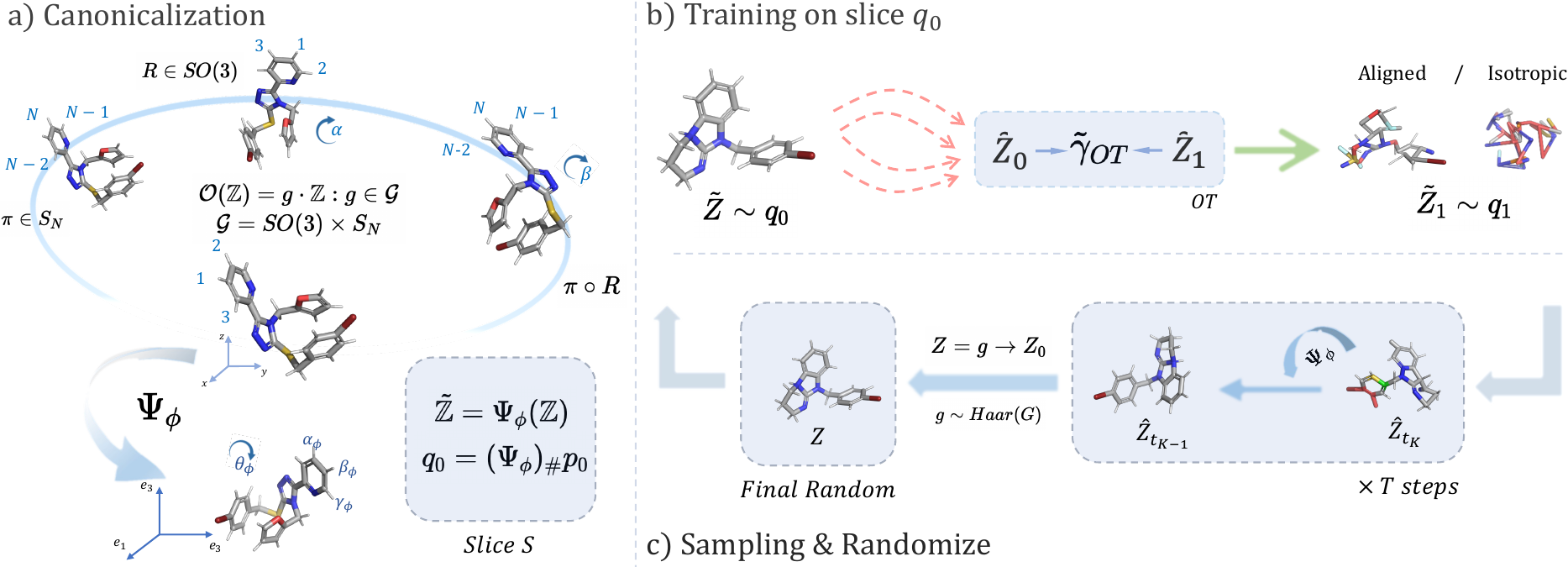}
    \caption{Overview of our canonicalized generation pipeline. (a) Canonicalization: map a molecule $Z$ to a slice representative $\tilde Z=\Psi_\phi(Z)$ under $\mathcal G=SO(3)\times S_N$, inducing $q_0=(\Psi_\phi)_\# p_0$. (b) Training: learn a diffusion/flow model on $\tilde Z_0\sim q_0$ with slice prior $q_1$ (optionally using an OT coupling for aligned pairings). (c) Sampling: generate on the slice (optionally with projected canonical sampling) and then apply $g\sim\mathrm{Haar}(\mathcal G)$ to obtain an invariant sample $Z=g\to \hat Z_0$.}
    \label{fig:canonical_pipeline}
    \vspace{-8pt}
\end{figure*}

\subsubsection{Conditional Flow Variance Reduction via Symmetry Ambiguity Elimination}\label{subsubsec:variance}
Here we formalize our central result: \emph{canonicalization reduces conditional variance in flow matching training} by eliminating symmetry ambiguity.

WLOG, consider the linear-path setting, the conditional vector field is constant:
\begin{equation}
\begin{aligned}
\Z_t = (1-t)\Z_0 + t\Z_1,\qquad (\Z_0,\Z_1)\sim \gamma,\\
u_t(\Z_t\mid \Z_0,\Z_1)=\frac{\mathrm{d}}{\mathrm{d}t}\Z_t = \Z_1-\Z_0.
\end{aligned}
\end{equation}
The optimal marginal field is $v_t(\Z)=\E[\Z_1-\Z_0\mid \Z_t=\Z]$.
Let $\hat v$ be any measurable predictor of $v_t(\Z_t)$ from $(t,\Z_t)$. Then the minimum achievable MSE (equivalently, the \emph{irreducible flow-matching error}) is the conditional variance:
\begin{equation*}
\inf_{\hat v}\E\big[\norm{\hat v(t,\Z_t)-(\Z_1-\Z_0)}^2\big]
=\E\big[\mathrm{Var}(\Z_1-\Z_0\mid t,\Z_t)\big].
\end{equation*}
We aim to understand when $\mathrm{Var}(Z_1-Z_0|t,Z_t)$ is large, especially at the presence of symmetry.
Consider the following \emph{symmetry-mixture data model}.
Let a compact group $\mathcal{G}$ act linearly and orthogonally on $\mathbb{R}^d$ (so $\|g\cdot x\|=\|x\|$). Assume the data distribution $p_0$ is $\mathcal{G}$-invariant:
\begin{equation}
\begin{aligned}
Z_0 \stackrel{d}{=} &G \cdot \widetilde Z_0\sim p_0,\ G\sim \lambda\ \text{(Haar)},\ \widetilde Z_0\sim q_0,
\\ &\widetilde Z_0 = \Psi(Z_0),\ q_0 = \Psi_{\#}p_0,
\end{aligned}
\end{equation}
where $\Psi$ is a (measurable) canonicalizer mapping each orbit to a representative on a slice. Intuitively, without canonicalization the posterior $p(Z_0\mid Z_t)$ can be multi-modal over the latent symmetry element $G$, inflating conditional variance. 

Compare the following two training paradigms:
(i) \textbf{canonical slice training} (quotient space),
train on canonicalized data $\widetilde Z_0:=\Psi(Z_0)\sim q_0$ with a slice prior $\widetilde Z_1\sim q_1$ and a coupling $\widetilde\gamma(\widetilde Z_0,\widetilde Z_1)$;
(ii) \textbf{invariant-mixture training} (ambient space),
train directly on the invariant data $Z_0\sim p_0$ with an ambient prior $Z_1\sim p_1$ and a coupling $\gamma(Z_0, Z_1)$.

To relate slice coupling to an invariant ambient-space coupling, we define the \emph{group-aligned lift} of a given $\widetilde\gamma$:
\begin{equation}
\label{eq:group-aligned-lift}
\begin{aligned}
G\sim\lambda,\ (\widetilde Z_0,\widetilde Z_1)\sim \widetilde\gamma\ \text{independent},\\
(Z_0,Z_1):=(G\cdot \widetilde Z_0,\; G\cdot \widetilde Z_1)\sim \gamma^{\mathrm{lift}}.
\end{aligned}
\end{equation}
We emphasize that the coupling correspondence (namely the invariant training uses the induced group-aligned lifting as the ambient coupling) is practically reasonable.
When $\tilde Z_1\sim \mathcal N(0, I)$ and the slice coupling is the independent product $\tilde \gamma:=q_0\otimes q_1$, then the ambient prior $Z_1\sim p_1=\mathcal{N}(0,I)$ and the invariant coupling recovers the widely adopted product coupling (w/o optimal transport) $\gamma_{mix}:=p_0\otimes p_1$; refer to Proposition~\ref{prop:lift-product} for more details.
This generally does not hold, however, if $\tilde Z_1$ is group-aware, or $\tilde Z_1$ and $\tilde Z_0$ are not independent; more details available in Appendix~\ref{subsubsec_appendix:proof_variance}.

\begin{remark}
All variance decompositions below compare \emph{under the lifted coupling} $\gamma^{\mathrm{lift}}$ (i.e., they isolate the effect of marginalizing over $G$ versus conditioning on $G$).
\end{remark}
We have the following sharp variance decomposition and conditional variance comparison under the lifted coupling.

\begin{theorem}[Variance decomposition under group-aligned lift]
\label{thm:canon-var-decomp}
Assume $\mathcal{G}$ acts orthogonally. Under the group-aligned lifted coupling
$G\sim\lambda$ independent of $(\widetilde Z_0,\widetilde Z_1)$ and $(Z_0,Z_1)=(G\cdot\widetilde Z_0,G\cdot\widetilde Z_1)$,
let $Z_t=(1-t)Z_0+tZ_1$ and $\widetilde Z_t=(1-t)\widetilde Z_0+t\widetilde Z_1$.
With $U:=Z_1-Z_0=G\cdot\Delta$ and $\Delta:=\widetilde Z_1-\widetilde Z_0$, we have
\begin{equation}\label{eq:canon-var-decomp}
\begin{aligned}
\mathrm{Var}(U\mid Z_t)&=
\underbrace{\mathbb{E}\!\left[\mathrm{Var}(\Delta\mid \widetilde Z_t)\mid Z_t\right]}_{\text{within-slice difficulty}}\\
&+\underbrace{\mathrm{Var}\!\left(\mathbb{E}[U\mid Z_t,G]\mid Z_t\right)}_{\text{symmetry ambiguity}\ge 0}.
\end{aligned}
\end{equation}
Consequently, invariant training without canonicalization admits larger conditional variance:
\begin{equation}\label{eq:canon-var-ineq}
\mathbb{E}\big[\mathrm{Var}(U\mid Z_t)\big]\;\ge\;\mathbb{E}\big[\mathrm{Var}(\Delta\mid \widetilde Z_t)\big].
\end{equation}
\end{theorem}
In particular, canonicalization eliminates the second term caused by symmetry ambiguity, and strict inequality typically holds whenever $G$ remains ambiguous given $Z_t$ and the conditional mean drift depends on $G$; see concrete posterior collision bounds in Appendix~\ref{subsubsec_appendix:proof_variance}.

\begin{remark}[Only non-equivariant models benefit from canonicalization]\label{remark:non-equi-benefit}
Notably, \Cref{thm:canon-var-decomp} only provides the lower bound of irreducibly error, but does not indicate all models are able to achieve this lower bound in practice. We emphasize that only non-equivariant models benefit from symmetry ambiguity elimination practically since canonicalization and canonical conditions expands the effective hypothesis class, echoing our analysis in \Cref{subsubsec:superior}. Equivariant models, by contrast, adhere the same ambient Bayes risk after canonicalization as in invariant training. More details are referred to Proposition~\ref{prop:bayes-equivariant} in the Appendix.
\end{remark}

\subsubsection{Within-Slice Simplification with Aligned Canonical Prior and Monge Coupling}\label{subsubsec:align_ot}
Theorem~\ref{thm:canon-var-decomp} isolates two distinct sources of irreducible error: within-slice difficulty and symmetry ambiguity.
Canonicalization reduces the second term, but the first term still depends on the slice coupling $\tilde \gamma(\tilde Z_0, \tilde Z_1)$ and can remain large when the slice coupling is far from straight transport. We hereby present techniques targeting the first term: \emph{aligned canonical priors} and \emph{optimal transport}. 

\textbf{Principles of aligned prior for canonical slice.}
We start from the choice of prior $\tilde Z_1\sim q_1$. Consider the case if one keeps a \emph{misaligned} simple prior (e.g.\ $q_1=\mathcal{N}(0,I)$ while $q_0$ on the canonical slice is far from centered/isotropic or is low-dimensional), then the slice coupling is still difficult and the first term in \Cref{thm:canon-var-decomp} can remain large; see Proposition~\ref{prop:gaussian-product-condvar} for an example with closed-form within-slice conditional variance.
Therefore, we propose two simple yet effective approaches to derive \textbf{aligned canonical priors}:
(i) \emph{learnable priors} within certain function classes (e.g. a Gaussian with learnable mean and variance); (ii) \emph{KL projections of $q_0$} onto certain function classes (e.g. moment-matched Gaussian). For example, among Gaussian canonical priors, the moment matched Gaussian is $q_1^\star\sim\mathcal N(\mathbb{E}_{q_0}[\widetilde Z_0], \mathrm{Cov}_{q_0}(\widetilde Z_0))$~(Proposition~\ref{prop:kl-gaussian-projection}).

\textbf{Optimal transport and near-Monge coupling as complementary.}
Given the canonical data $q_0$ and the prior $q_1$, the slice conditional variance $\mathbb{E}[\mathrm{Var}(\widetilde Z_1-\widetilde Z_0\mid \widetilde Z_t)]$ still depends on the coupling $\tilde \gamma$. In favorable regimes, e.g.\ Monge-like (near) optimal transport (OT) under standard regularity conditions, this term can be small.

Remarkably, OT has been applied in some previous work~\citep{irwin2024semlaflow}. We now formally clarify its effectiveness on symmetry space, either with or without canonicalization. While OT generally reduces trajectory intersections and conditional variances regardless of whether canonicalization is applied, OT solutions are not unique in the presence of symmetry.
Averaging yields a diagonal-invariant optimal coupling, but may destroy determinism (which turns a Monge map into a mixture).
Fortunately, canonicalization stabilizes the solution by providing a gauge, and further benefits non-equivariant models taking canonical conditions as inputs, as will be illustrated in Appendix~\ref{subsubsec_appendix:proof_variance}. Alternatively, canonicalization serves as an ``implicit OT" by aligning the data-noise pair into the same gauge.
\begin{remark}[Canonicalization and optimal transport act complementarily]
    In the presence of symmetry, canonicalization eliminates group ambiguity, while optimal transport reduces within-slice difficulty. OT solutions can be closer to unique/Monge-like with canonicalization (\Cref{fig:ot-canon-complementary}).
\end{remark}
\textbf{Training--sampling consistency.}
Suppose the network is conditioned on an auxiliary variable $C$ (e.g.\ ``canonical rank'' positional encoding, a frame ID, or any deterministic/stochastic output of a canonicalizer).
Training then implicitly targets a conditional field $\widetilde v(t,z,c)$ under a joint law $\tilde\pi_t(c,z)$ induced by the training pipeline.
We say there is \emph{no condition mismatch at the start} if the inference-time joint law $\tilde\pi^{\mathrm{inf}}_1(C,\widetilde Z_1)$ equals the training-time joint law $\tilde\pi^{\mathrm{tr}}_1(C,\widetilde Z_1)$.
With either (i) aligned canonical priors, or (ii) isotropic priors and training-time OT, one keeps the training--sampling consistency. Indeed, the training trajectories in \Cref{fig:train_traj} strongly support our claim that canonicalization (especially when combined with aligned canonical prior) significantly accelerates training.

\section{Canonical Diffusion for Molecular Graph Generation}\label{sec:method}

Building on the theoretical foundations established in \Cref{sec:theory}, we now instantiate our canonical molecular graph diffusion framework under the joint symmetry group $\G = S_N \times SO(3)$ (atom index permutations and global rotations).

\subsection{Canonicalization of Molecular Graphs}
\label{subsec:method-canon}
The key design choice lies in constructing a canonicalizer that is both \emph{geometrically principled} and \emph{computationally tractable}. We decompose the full symmetry group, addressing permutation and rotation symmetries sequentially:
\begin{equation}
\Psi_\phi(\Z) \;=\; \Psi^{(\mathrm{rot})}_\phi\!\left(\Psi^{(\mathrm{perm})}_\phi(\Z)\right),
\end{equation}
where $\Psi^{(\mathrm{perm})}_\phi$ establishes a canonical atom ordering and $\Psi^{(\mathrm{rot})}_\phi$ fixes a global reference frame.

\begin{figure}[t]
    \centering
    \includegraphics[width=\linewidth]{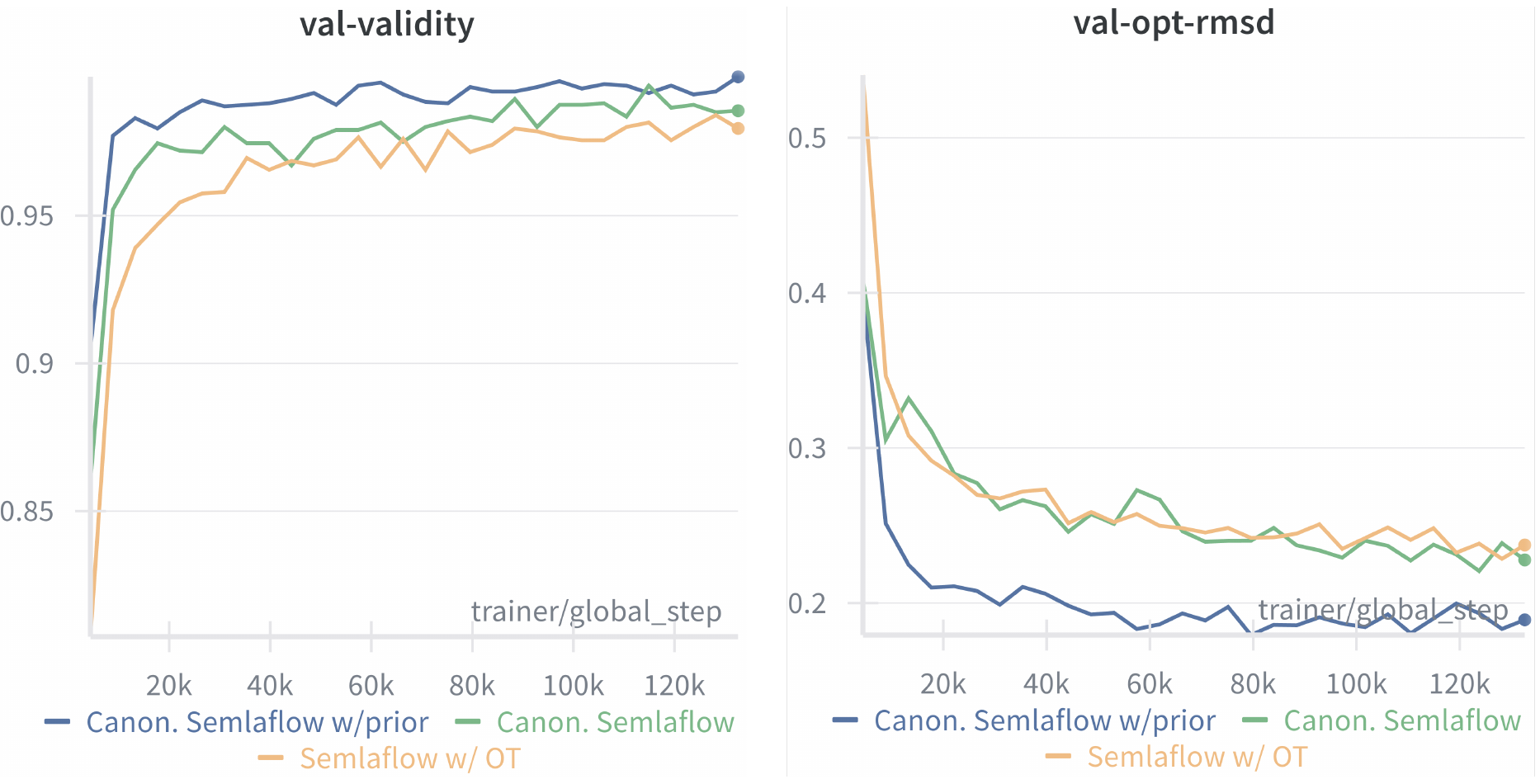}
    \caption{Training trajectories of our canonicalized model (visualization of the learned transport dynamics).}
    \label{fig:train_traj}
    \vspace{-0.3cm}
\end{figure}

We consider \textbf{canonicalization via geometric spectra}, which breaks both permutation and rotation symmetry in a coordinated manner. 
The Fiedler vector, i.e. the eigenvector corresponding to the second smallest eigenvalue (the algebraic connectivity) of a graph’s Laplacian matrix, determines a canonical atom ordering in a natural way.
Concretely, given a molecular, we build a symmetric kernel matrix $\W(\Xc,\A)\in\R^{N\times N}$ from pairwise distances (e.g., $W_{ij}=k(\|\Xc_i-\Xc_j\|)$ for a radial kernel $k$) to handle the 3D geometry, forming the normalized geometric Laplacian 
\begin{equation}
\mathbf L(\Xc,\A)\;=\;\D^{-1}(\D-\W), \quad \D=\mathrm{diag}(\W\1).
\end{equation}
$\mathbf L$ depends only on distances and is rotation-invariant. 
We then use the Fiedler vector $\mathbf{u_2}\in\R^N$ to define a canonical ordering
\begin{equation}
\pi^\star(\Z)\;:=\;\mathrm{argsort}(\mathbf{u_2})\in S_N.
\end{equation}
Under the standard mass–spring interpretation of the Laplacian, this corresponds to the lowest-frequency non-rigid vibration mode~\cite{knyazev2018spectralpartitioningsignedgraphs}. This rank yields the classical spectral bisection and provides a principled seriation keeping strongly connected substructures contiguous~\cite{doi:10.1137/0611030}. In the molecular context, the spectral ordering serves as a \emph{geometric-aware linearization}. 

Finally, we optionally define a global $SO(3)$ gauge in a rank-consistent way after applying $\pi^\star$. We choose a small set of anchor atoms in the canonical order (e.g., the two extremes define a primary axis and a third anchor determines a plane normal), and rotate the centered coordinates into the resulting right-handed orthonormal frame. This yields a deterministic representative under the joint $S_N\times SO(3)$ action; the full algorithm is deferred to Algorithm~\ref{alg:spectral-so3-canon}.

There are also other existing canonicalization methods, mostly for permutation in abstract graphs~\citep{zhao2024pard,LaplacianCanonization,RethinkingGraphCanon} instead of geometric graphs, while existing OT methods mostly target coordinates; our method consistently consider the joint group in a stable manner. 
Different canonicalization methods may lead to various stability and practical performance, yet we experimentally find that our novel geometric spectra canonicalization outperforms others (Appendix~\ref{subsec_appendix:canon}).

\subsection{Canonical Diffusion and CanonFlow}\label{subsec:algorithm}

\paragraph{Canonicality-aware denoising architectures.} We elaborate on two types of architectural changes to enable non-equivariant denoising model in canonical diffusion: (i) \emph{positional encoding only}, where we keep the original model architectures except that the model is conditioned on canonical rank embeddings $C=\varphi(r_i)$ with normalized rank $r_i=(i-1)/(N-1)$ to handle various molecular sizes; (ii) \emph{\textbf{Canon} architecture}, our newly designed architecture that explicitly keeps atom canonical embeddings as a learnable hidden state $\mathbf C\in\mathbb R^{N\times d_c}$, which interacts with the embeddings of $\Xc,\Hc,\A$ in each layer (\Cref{fig:canon-arch}); see more details in Appendix~\ref{subsec_appendix:arch}. 
We use ``\emph{Canon. BM}" to indicate the first type, where everything is consistent with the base model ``\emph{BM}" except the PE; ``\emph{CanonFlow}" specifically refers to the flow matching model with \emph{Canon} architecture, whose other settings are aligned with SemlaFlow.

\paragraph{Training on the canonical slice.}

Training proceeds via standard diffusion or flow matching procedures on the canonical slice $\widetilde\Z=\Psi_\phi(\Z)$ and the slice distribution $q_0=(\Psi_\phi)\push p_0$. For convenience, we illustrate with the flow matching framework, and diffusion models follow analogously. Given a coupling $\widetilde\gamma(\widetilde\Z_0,\widetilde\Z_1)$ between canonicalized data $\widetilde\Z_0\sim q_0$ and the (optionally aligned) prior $\widetilde\Z_1\sim q_1$, we construct interpolated paths $\widetilde\Z_t=\Phi_t(\widetilde\Z_0,\widetilde\Z_1)$ for $t\sim\mathrm{Unif}[0,1]$. 
The model $m_\theta$ learns to predict the conditional velocity field:
\begin{equation}
\min_\theta\;
\E\Big[\big\|m_\theta(\widetilde\Z_t,t;\,C)-u_t(\widetilde\Z_t\mid \widetilde\Z_0,\widetilde\Z_1)\big\|^2\Big],
\end{equation}
where $u_t = \frac{\mathrm{d}}{\mathrm{d}t}\Phi_t$ including velocities for both continuous and discrete features. 
As explained in \Cref{subsubsec:align_ot}, we also optionally adopt an aligned prior, or enable OT coupling early in training and anneal its usage over epochs (termed as \emph{OT anneal}) to stabilize optimization while avoiding over-specialization to aligned pairings. 

\paragraph{Sampling and restoring invariance.}
Starting from the slice prior $\widetilde\Z_1\sim q_1$ (the same marginal prior used in training), we integrate the learned dynamics to obtain a canonical sample $\widehat{\Z}_0$ on the slice $S$; no OT matching is required at inference even if OT is used during training. 
Denote the sampled canonical distribution as $\tilde \mu$, we can then simply restore the invariance by sampling $g\sim\Haar$, $\tilde\Z\sim\tilde\mu$ and outputting $g\acts \tilde\Z$, thus the randomized distribution $\mu := \int_{\G} (g\acts)\push \tilde\mu \,\mathrm{d}\Haar(g)$ is invariant (Proposition~\ref{prop:randomize}).

\begin{table*}[t]
\centering
\caption{Performance of canonicalized SemlaFlow on QM9 dataset. Atom stability, molecule stability, and validity are reported as percentages; Opt-RMSD is reported in \AA. 
}
\vspace{-0.8mm}
\small
\label{tab:unconditional_qm9}
\begin{tabular}{l|cccc|c}
\toprule
Methods              & Atom Stab $\uparrow$        & Mol Stab $\uparrow$            & Valid $\uparrow$        & Opt-RMSD $\downarrow$ & NFE $\downarrow$     \\
\midrule

EDM   & 98.7 & 82.0 & 91.9& -- & 1000   \\
GCDM & 98.7 & 85.7 & 94.8 & -- & 1000\\
MUDiff & 98.3 & 89.9 & 95.3 & -- & 1000 \\
\midrule

FlowMol &  99.7 & 96.2 & 97.3 & -- & 100\\
MiDi               & 99.8                  & 97.5                & 97.9              & --              & 500        \\
EQGAT-diff         & \pmerr{99.9}{0.0}           & \pmerr{98.7}{0.18}           & \pmerr{99.0}{0.16}       & -- & 500      \\
Semlaflow w/ OT & \pmerr{99.9}{0.0} & \pmerr{99.56}{0.07} & \pmerr{99.31}{0.08} &  \pmerr{0.23}{0.0} & 100\\
\midrule

Canon. SemlaFlow w/ OT anneal + PCS (ours) & \pmerr{99.9}{0.0}          & \pmerr{99.62}{0.08}           &  \pmerr{99.49}{0.06}     & \pmerr{0.21}{0.0} & 100  \\ 
Canon. SemlaFlow w/ Prior + PCS (ours) & \pmerr{99.9}{0.0}          & \pmerr{\textbf{99.64}}{0.04}           &  \pmerr{\textbf{99.58}}{0.06}     & \pmerr{\textbf{0.17}}{0.0} & 100  \\  
\bottomrule
\end{tabular}
\end{table*}

\begin{table*}[t]
\centering
\caption{Performance of canonicalized SemlaFlow and CanonFlow on GEOM-DRUG dataset. Atom stability, molecule stability, validity, as well as uniqueness and novelty are reported as percentages.
}
\vspace{-0.8mm}
\label{tab:unconditional_drug}
\small
\begin{tabular}{l|ccccc|c}
\toprule
Methods              & Atom Stab $\uparrow$        & Mol Stab $\uparrow$            & Valid $\uparrow$       & Unique $\uparrow$   & Novel $\uparrow$ & NFE $\downarrow$     \\
\midrule
FlowMol &  99.0 & 67.5 & 51.2  & - & - & 100\\
MiDi               & 99.8                  & 91.6                & 77.8                   & 100.0              & 100.0        & 500        \\
EQGAT-diff         & \pmerr{99.8}{0.0}           & \pmerr{93.4}{0.21}           & \pmerr{94.6}{0.24}        & \pmerr{100.0}{0.0} & \pmerr{99.9}{0.07} & 500      \\
SemlaFlow w/ OT      & \pmerr{99.8}{0.0}          & \pmerr{97.3}{0.08}           &  \pmerr{93.9}{0.19}      & \pmerr{100.0}{0.0} & \pmerr{99.6}{0.03}   & 100  \\ 
\midrule
Canon. SemlaFlow w/ OT anneal (ours) & \pmerr{99.8}{0.0}          & \pmerr{98.1}{0.03}           &  \pmerr{95.0}{0.20}       & \pmerr{100.0}{0.0} &  \pmerr{99.7}{0.02}   & 100  \\ 
CanonFlow w/ OT anneal (ours) & \pmerr{\textbf{99.9}}{0.0}          & \pmerr{\textbf{98.4}}{0.02}           &  \pmerr{\textbf{95.9}}{0.08}       & \pmerr{100.0}{0.0} &  \pmerr{99.7}{0.01}   & 100  \\ 
\bottomrule
\end{tabular}
\vspace{-0.8mm}
\end{table*}

\paragraph{Improved techniques with canonicality.}
A subtle discrepancy arises between training and inference: canonical ranks are computed exactly from $\Psi_\phi(\Z)$ during training, whereas at inference the model generates from noise without access to ground-truth ranks.
Moreoever, in the spirit of \citep{dym2024equivariant}, continuous canonicalization is impossible for $S_N$ and $SO(d)$. To address this issue and to mitigate the train-test gap, we relaxed the ranks into continuous values and inject small perturbations to rank features during training, with an optional auxiliary head that self-estimates ranks from intermediate representations. In sampling, models either utilize fixed canonical conditions, or adaptively re-estimate the ranks to align the noisy data and the canonical conditions, which we termed as \emph{projected canonical sampling} (PCS). Inspired by \cite{li2024geometric}, we also randomly drop canonical conditions with probability $p_{\text{drop}}$ during training, and enable classifier-free guidance (CFG)~\citep{ho2022classifier} extrapolating conditional and unconditional generation during inference.
More training and sampling techniques are deferred to Appendix~\ref{sec_appendix:imple_details}.

\section{Experiments}\label{sec:exp}

\paragraph{Experimental Setup.} 
We evaluate our canonicalization approach on two widely adopted benchmarks for unconditional 3D molecule generation: QM9~\citep{ramakrishnan2014quantum} and GEOM-DRUG~\citep{axelrod2022geom}. 
While QM9 consists of relatively small molecules, GEOM-DRUG provides a more challenging testbed that better differentiates the capabilities of various generative models.
We adopt identical train/validation/test splits as prior work~\citep{vignac2023midi, le2023navigating}.
Following~\citep{irwin2024semlaflow}, we exclude molecules containing more than 72 atoms from the GEOM-DRUG training set for computational efficiency, which accounts for approximately 1\% of the data; the validation and test sets remain intact.
During evaluation, we sample molecule sizes from the empirical distribution in the test set and generate the corresponding number of atoms by numerically integrating the learned flow dynamics.
All reported metrics are computed over 10k generated molecules with 3 independent runs. Without specification, we adopt the settings and hyper-parameters in SemlaFlow~\citep{irwin2024semlaflow}, which serves as our base model. By default, canonicalization is conducted on the $S_N$ group; results of canonicalizing $S_N\times SO(3)$ as well as more ablation studies are deferred to Appendix~\ref{sec_appendix:exp}. 

\vspace{-1pt}
\paragraph{Main results.} 
Applying the canonical diffusion framework to SemlaFlow and CanonFlow, we achieve state-of-the-art results on most metrics across both datasets.

\Cref{tab:unconditional_qm9} shows results on QM9.
While all methods achieve near-saturated atom stability, our canonicalized variants improve upon SemlaFlow in both molecule stability and validity.
Notably, Canon.\ SemlaFlow (Prior + PCS) achieves the lowest Opt-RMSD of 0.17\AA, a 26\% reduction compared to the baseline (0.23\AA), indicating that canonicalization guides the model toward generating geometries closer to energy-minimized conformations.
In addition, we report the training dynamics on QM9 in \Cref{fig:train_traj}. One could observe that with canonicalization, models achieve much higher validity and lower opt-RMSD in validation within the same epochs compared with baseline, and aligned canonical prior further increases the advantages, verifying that canonicalization accelerates training convergence.

\Cref{tab:unconditional_drug} presents results on the more challenging GEOM-DRUG dataset, which contains larger and structurally diverse drug-like molecules that better differentiate model capabilities.
On this benchmark, our canonicalized approach demonstrates clear improvements over baselines.
In particular, Canon.\ SemlaFlow (trained w/ CFG) consistently improves validity and molecule stability over SemlaFlow baseline. Remarkably, our CanonFlow achieves \emph{SOTA} performance across atom stability, molecule stability and validity -- the most important metrics, outperforming all previous models with a large margin (surpassing SemlaFlow baseline $1.1\%$ in molecule stability and $2.0\%$ in validity).

We also report few-step generation performance in \Cref{table:few_step}. The advantages of canonicalized models are still significant with even only 50 steps, which is only ten percent of most previous models. Notably, our method almost induces no overheads in computation or sample time. This additionally validates the efficient and effective guidance signal of canonical conditioning in sampling.

\begin{table}
  \caption{Comparison of few-step generation on GEOM-DRUG (w/o CFG in training and sampling). Sample time is measured by the average number of seconds to generate 1000 molecules. Canonical SemlaFlow significantly outperforms baselines w/ fewer sampling steps, with negligible increased sample time.}
  \label{table:few_step}
  \centering
  \resizebox{1.0\columnwidth}{!}{
  \begin{tabular}{lcccc}
    \toprule
    Model              & Mol Stab $\uparrow$  & Valid $\uparrow$     & Sample Time $\downarrow$  & NFE  $\downarrow$  \\
    \midrule
    EQGAT-diff         & 93.4$_{\pm \text{0.21}}$  & 94.6$_{\pm \text{0.24}}$
        &  2293.0  & 500  \\
    \midrule
    \text{SemlaFlow}$_{50}$ w/ OT    & 97.0$_{\pm \text{0.21}}$  & 93.9$_{\pm \text{0.12}}$  & 49.8    & 50   \\
    \cellcolor{SkyBlue!20}\text{Canon. SemlaFlow}$_{50}$     & \pmerr{97.5}{0.14}  & \pmerr{94.9}{0.07} & 50.1    & 50   \\
    \midrule
    \text{SemlaFlow}$_{100}$ w/ OT   & 97.3$_{\pm \text{0.08}}$  & 93.9$_{\pm \text{0.19}}$  & 99.3    & 100  \\
    \cellcolor{SkyBlue!20}\text{Canon. SemlaFlow}$_{100}$     & \pmerr{97.9}{0.09}  & \pmerr{94.9}{0.18}  & 99.7    & 100   \\
    \midrule
    Data              & 100.0  & 100.0  & --  & -- \\
    \bottomrule
  \end{tabular}
  }
  \vspace{-0.5cm}
\end{table}

\section{Conclusion}
We studied symmetry in diffusion/flow-based generative modeling through the lens of canonicalization and transport. For molecular data, where representations are invariant to permutations and rigid motions, we showed that symmetry can induce latent \emph{group ambiguity} at intermediate noise levels, inflating conditional variance and making few-step solvers less reliable. Canonicalization provides a principled gauge fixing: it collapses each orbit to a unique slice representative, removing the symmetry-induced ambiguity term and enabling the use of expressive non-equivariant backbones with stable, fixed indexing semantics (e.g., rank-based inputs). The remaining within-slice  difficulty can be mitigated by better aligned priors and near-Monge couplings during training. Together, these results clarify when and why canonicalization and OT act complementarily, how to avoid projection- or conditioning-induced mismatches, and how to design practical, scalable symmetry-aware generators for molecules that improve few-step quality while preserving invariance by construction. Canonical diffusion yields state-of-the-art results across several metrics on popular geometric molecule generation benchmarks, with significant advantages in training convergence.

\section*{Acknowledgements}

Cai Zhou would like to sincerely acknowledge Prof. Yusu Wang from UCSD, Prof. Stefanie Jegelka from MIT and TU Munich, Xingang Peng and Xiyuan Wang from Peking University for their constructive discussions.

\section*{Impact Statement}
This paper presents work whose goal is to advance the field of Machine Learning, Generative Models and Molecule Design. There are some potential societal consequences of our work, none which we feel must be specifically highlighted here.

\bibliography{ref}

@inproceedings{vignac2023midi,
  title={Midi: Mixed graph and 3d denoising diffusion for molecule generation},
  author={Vignac, Clement and Osman, Nagham and Toni, Laura and Frossard, Pascal},
  booktitle={Joint European Conference on Machine Learning and Knowledge Discovery in Databases},
  pages={560--576},
  year={2023},
  organization={Springer}
}

@article{axelrod2022geom,
  title={GEOM, energy-annotated molecular conformations for property prediction and molecular generation},
  author={Axelrod, Simon and Gomez-Bombarelli, Rafael},
  journal={Scientific Data},
  volume={9},
  number={1},
  pages={185},
  year={2022},
  publisher={Nature Publishing Group UK London}
}

@inproceedings{xu2023geometric,
  title={Geometric latent diffusion models for 3d molecule generation},
  author={Xu, Minkai and Powers, Alexander S and Dror, Ron O and Ermon, Stefano and Leskovec, Jure},
  booktitle={International Conference on Machine Learning},
  pages={38592--38610},
  year={2023},
  organization={PMLR}
}

@article{zhao2024pard,
  title={Pard: Permutation-invariant autoregressive diffusion for graph generation},
  author={Zhao, Lingxiao and Ding, Xueying and Akoglu, Leman},
  journal={Advances in Neural Information Processing Systems},
  volume={37},
  pages={7156--7184},
  year={2024}
}

@inproceedings{hoogeboom2022equivariant,
  title={Equivariant diffusion for molecule generation in 3d},
  author={Hoogeboom, Emiel and Satorras, V{\i}ctor Garcia and Vignac, Cl{\'e}ment and Welling, Max},
  booktitle={International conference on machine learning},
  pages={8867--8887},
  year={2022},
  organization={PMLR}
}

@article{ramakrishnan2014quantum,
  title={Quantum chemistry structures and properties of 134 kilo molecules},
  author={Ramakrishnan, Raghunathan and Dral, Pavlo O and Rupp, Matthias and Von Lilienfeld, O Anatole},
  journal={Scientific data},
  volume={1},
  number={1},
  pages={1--7},
  year={2014},
  publisher={Nature Publishing Group}
}

@inproceedings{satorras2021n,
  title={E (n) equivariant graph neural networks},
  author={Satorras, V{\i}ctor Garcia and Hoogeboom, Emiel and Welling, Max},
  booktitle={International conference on machine learning},
  pages={9323--9332},
  year={2021},
  organization={PMLR}
}

@article{garcia2021n,
  title={E (n) equivariant normalizing flows},
  author={Garcia Satorras, Victor and Hoogeboom, Emiel and Fuchs, Fabian and Posner, Ingmar and Welling, Max},
  journal={Advances in Neural Information Processing Systems},
  volume={34},
  pages={4181--4192},
  year={2021}
}

@article{wu2022diffusion,
  title={Diffusion-based molecule generation with informative prior bridges},
  author={Wu, Lemeng and Gong, Chengyue and Liu, Xingchao and Ye, Mao and Liu, Qiang},
  journal={Advances in Neural Information Processing Systems},
  volume={35},
  pages={36533--36545},
  year={2022}
}

@inproceedings{rombach2022high,
  title={High-resolution image synthesis with latent diffusion models},
  author={Rombach, Robin and Blattmann, Andreas and Lorenz, Dominik and Esser, Patrick and Ommer, Bj{\"o}rn},
  booktitle={Proceedings of the IEEE/CVF conference on computer vision and pattern recognition},
  pages={10684--10695},
  year={2022}
}

@inproceedings{you2023latent,
  title={Latent 3D Graph Diffusion},
  author={You, Yuning and Zhou, Ruida and Park, Jiwoong and Xu, Haotian and Tian, Chao and Wang, Zhangyang and Shen, Yang},
  booktitle={The Twelfth International Conference on Learning Representations},
  year={2023}
}

@inproceedings{li2024distance,
 author = {Li, Zian and Wang, Xiyuan and Huang, Yinan and Zhang, Muhan},
 booktitle = {Advances in Neural Information Processing Systems},
 pages = {37413--37447},
 title = {Is Distance Matrix Enough for Geometric Deep Learning?},
 volume = {36},
 year = {2023}
}

@article{li2024completeness,
  title={On the Completeness of Invariant Geometric Deep Learning Models},
  author={Li, Zian and Wang, Xiyuan and Kang, Shijia and Zhang, Muhan},
  journal={arXiv preprint arXiv:2402.04836},
  year={2024}
}

@article{ho2022classifier,
  title={Classifier-free diffusion guidance},
  author={Ho, Jonathan and Salimans, Tim},
  journal={arXiv preprint arXiv:2207.12598},
  year={2022}
}

@article{xu2022geodiff,
  title={Geodiff: A geometric diffusion model for molecular conformation generation},
  author={Xu, Minkai and Yu, Lantao and Song, Yang and Shi, Chence and Ermon, Stefano and Tang, Jian},
  journal={arXiv preprint arXiv:2203.02923},
  year={2022}
}

@article{dhariwal2021diffusion,
  title={Diffusion models beat gans on image synthesis},
  author={Dhariwal, Prafulla and Nichol, Alexander},
  journal={Advances in neural information processing systems},
  volume={34},
  pages={8780--8794},
  year={2021}
}

@article{song2024equivariant,
  title={Equivariant flow matching with hybrid probability transport for 3d molecule generation},
  author={Song, Yuxuan and Gong, Jingjing and Xu, Minkai and Cao, Ziyao and Lan, Yanyan and Ermon, Stefano and Zhou, Hao and Ma, Wei-Ying},
  journal={Advances in Neural Information Processing Systems},
  volume={36},
  year={2023}
}

@article{HowPowerfulGNN,
  title={How Powerful are Graph Neural Networks?},
  author={Keyulu Xu and Weihua Hu and Jure Leskovec and Stefanie Jegelka},
  journal={ArXiv},
  year={2018},
  volume={abs/1810.00826}
}

@inproceedings{morris2019weisfeiler,
  title={Weisfeiler and leman go neural: Higher-order graph neural networks},
  author={Morris, Christopher and Ritzert, Martin and Fey, Matthias and Hamilton, William L and Lenssen, Jan Eric and Rattan, Gaurav and Grohe, Martin},
  booktitle={Proceedings of the AAAI conference on artificial intelligence},
  pages={4602--4609},
  year={2019}
}

@article{morris2020weisfeiler,
  title={Weisfeiler and Leman go sparse: Towards scalable higher-order graph embeddings},
  author={Morris, Christopher and Rattan, Gaurav and Mutzel, Petra},
  journal={Advances in Neural Information Processing Systems},
  volume={33},
  pages={21824--21840},
  year={2020}
}

@InProceedings{klWL,
  title = 	 {From Relational Pooling to Subgraph {GNN}s: A Universal Framework for More Expressive Graph Neural Networks},
  author =       {Zhou, Cai and Wang, Xiyuan and Zhang, Muhan},
  booktitle = 	 {Proceedings of the 40th International Conference on Machine Learning},
  pages = 	 {42742--42768},
  year = 	 {2023},
  volume = 	 {202},
  series = 	 {Proceedings of Machine Learning Research},
  publisher =    {PMLR}
}

@article{Kim2023pure,
  title={Pure Transformers are Powerful Graph Learners},
  author={Jinwoo Kim and Tien Dat Nguyen and Seonwoo Min and Sungjun Cho and Moontae Lee and Honglak Lee and Seunghoon Hong},
  journal={ArXiv},
  year={2022},
  volume={abs/2207.02505}
}

@article{Exphormer,
  title={Exphormer: Sparse Transformers for Graphs},
  author={Hamed Shirzad and Ameya Velingker and B. Venkatachalam and Danica J. Sutherland and Ali Kemal Sinop},
  journal={ArXiv},
  year={2023},
  volume={abs/2303.06147}
}

@article{LaplacianCanonization,
  title={Laplacian Canonization: A Minimalist Approach to Sign and Basis Invariant Spectral Embedding},
  author={Ma, Jiangyan and Wang, Yifei and Wang, Yisen},
  journal={arXiv preprint arXiv:2310.18716},
  year={2023}
}

@article{SignBasisNet,
  title={Sign and basis invariant networks for spectral graph representation learning},
  author={Lim, Derek and Robinson, Joshua and Zhao, Lingxiao and Smidt, Tess and Sra, Suvrit and Maron, Haggai and Jegelka, Stefanie},
  journal={arXiv preprint arXiv:2202.13013},
  year={2022}
}

@article{Specformer,
  title={Specformer: Spectral Graph Neural Networks Meet Transformers},
  author={Deyu Bo and Chuan Shi and Lele Wang and Renjie Liao},
  journal={ArXiv},
  year={2023},
  volume={abs/2303.01028}
}

@article{RWSELearnable,
  title={Graph Neural Networks with Learnable Structural and Positional Representations},
  author={Vijay Prakash Dwivedi and Anh Tuan Luu and Thomas Laurent and Yoshua Bengio and Xavier Bresson},
  journal={ArXiv},
  year={2021},
  volume={abs/2110.07875}
}

@inproceedings{song2020score,
  title={Score-Based Generative Modeling through Stochastic Differential Equations},
  author={Yang Song and Jascha Sohl-Dickstein and Diederik P Kingma and Abhishek Kumar and Stefano Ermon and Ben Poole},
  booktitle={International Conference on Learning Representations},
  year={2021},
  url={https://openreview.net/forum?id=PxTIG12RRHS}
}

@article{gebauer2019symmetry,
  title={Symmetry-adapted generation of 3d point sets for the targeted discovery of molecules},
  author={Gebauer, Niklas and Gastegger, Michael and Sch{\"u}tt, Kristof},
  journal={Advances in neural information processing systems},
  volume={32},
  year={2019}
}

@inproceedings{luo2022autoregressive,
  title={An autoregressive flow model for 3d molecular geometry generation from scratch},
  author={Luo, Youzhi and Ji, Shuiwang},
  booktitle={International conference on learning representations (ICLR)},
  year={2022}
}

@article{lipman2022flow,
  title={Flow matching for generative modeling},
  author={Lipman, Yaron and Chen, Ricky TQ and Ben-Hamu, Heli and Nickel, Maximilian and Le, Matt},
  journal={arXiv preprint arXiv:2210.02747},
  year={2022}
}

@article{liu2022flow,
  title={Flow straight and fast: Learning to generate and transfer data with rectified flow},
  author={Liu, Xingchao and Gong, Chengyue and Liu, Qiang},
  journal={arXiv preprint arXiv:2209.03003},
  year={2022}
}

@article{chen2023efficient,
  title={Efficient and degree-guided graph generation via discrete diffusion modeling},
  author={Chen, Xiaohui and He, Jiaxing and Han, Xu and Liu, Li-Ping},
  journal={arXiv preprint arXiv:2305.04111},
  year={2023}
}

@article{qin2024defog,
  title={Defog: Discrete flow matching for graph generation},
  author={Qin, Yiming and Madeira, Manuel and Thanou, Dorina and Frossard, Pascal},
  journal={arXiv preprint arXiv:2410.04263},
  year={2024}
}

@article{luo2023fast,
  title={Fast graph generation via spectral diffusion},
  author={Luo, Tianze and Mo, Zhanfeng and Pan, Sinno Jialin},
  journal={IEEE Transactions on Pattern Analysis and Machine Intelligence},
  year={2023},
  publisher={IEEE}
}

@inproceedings{jang2023hierarchical,
  title={Hierarchical Graph Generation with $K^2$-trees},
  author={Jang, Yunhui and Kim, Dongwoo and Ahn, Sungsoo},
  booktitle={ICML 2023 Workshop on Structured Probabilistic Inference $\{$$\backslash$\&$\}$ Generative Modeling},
  year={2023}
}

@article{vignac2022digress,
  title={Digress: Discrete denoising diffusion for graph generation},
  author={Vignac, Clement and Krawczuk, Igor and Siraudin, Antoine and Wang, Bohan and Cevher, Volkan and Frossard, Pascal},
  journal={arXiv preprint arXiv:2209.14734},
  year={2022}
}

@article{ho2020denoising,
  title={Denoising Diffusion Probabilistic Models},
  author={Jonathan Ho and Ajay Jain and P. Abbeel},
  journal={ArXiv},
  year={2020},
  volume={abs/2006.11239}
}

@article{latentgraphdiffusion,
  title={Unifying Generation and Prediction on Graphs with Latent Graph Diffusion},
  author={Cai Zhou and Xiyuan Wang and Muhan Zhang},
  journal={Advances in Neural Information Processing Systems},
  volume={37},
  year={2024}
}

@inproceedings{zhou2023hodge,
 author = {Zhou, Cai and Wang, Xiyuan and Zhang, Muhan},
 booktitle = {Advances in Neural Information Processing Systems},
 pages = {16172--16206},
 title = {Facilitating Graph Neural Networks with Random Walk on Simplicial Complexes},
 volume = {36},
 year = {2023}
}

@article{le2023navigating,
  title={Navigating the design space of equivariant diffusion-based generative models for de novo 3d molecule generation},
  author={Le, Tuan and Cremer, Julian and No{\'e}, Frank and Clevert, Djork-Arn{\'e} and Sch{\"u}tt, Kristof},
  journal={arXiv preprint arXiv:2309.17296},
  year={2023}
}

@inproceedings{campbell2022discrete,
  title={A Continuous Time Framework for Discrete Denoising Models},
  author={Andrew Campbell and Joe Benton and Valentin De Bortoli and Tom Rainforth and George Deligiannidis and Arnaud Doucet},
  booktitle={Advances in Neural Information Processing Systems},
  editor={Alice H. Oh and Alekh Agarwal and Danielle Belgrave and Kyunghyun Cho},
  year={2022},
  url={https://openreview.net/forum?id=DmT862YAieY}
}

@article{dym2024equivariant,
  title={Equivariant frames and the impossibility of continuous canonicalization},
  author={Dym, Nadav and Lawrence, Hannah and Siegel, Jonathan W},
  journal={arXiv preprint arXiv:2402.16077},
  year={2024}
}

@article{yan2023swingnn,
  title={Swingnn: Rethinking permutation invariance in diffusion models for graph generation},
  author={Yan, Qi and Liang, Zhengyang and Song, Yang and Liao, Renjie and Wang, Lele},
  journal={arXiv preprint arXiv:2307.01646},
  year={2023}
}

@InProceedings{kaba2023canonical,
  title = 	 {Equivariance with Learned Canonicalization Functions},
  author =       {Kaba, S\'{e}kou-Oumar and Mondal, Arnab Kumar and Zhang, Yan and Bengio, Yoshua and Ravanbakhsh, Siamak},
  booktitle = 	 {Proceedings of the 40th International Conference on Machine Learning},
  pages = 	 {15546--15566},
  year = 	 {2023},
  volume = 	 {202},
  series = 	 {Proceedings of Machine Learning Research},
  month = 	 {23--29 Jul},
  publisher =    {PMLR},
  url = 	 {https://proceedings.mlr.press/v202/kaba23a.html}
}

@inproceedings{tahmasebi2025generalization,
  title={Generalization Bounds for Canonicalization: A Comparative Study with Group Averaging},
  author={Behrooz Tahmasebi and Stefanie Jegelka},
  booktitle={The Thirteenth International Conference on Learning Representations},
  year={2025},
  url={https://openreview.net/forum?id=n0lXaskyk5}
}

@article{irwin2024semlaflow,
  title={SemlaFlow--Efficient 3D Molecular Generation with Latent Attention and Equivariant Flow Matching},
  author={Irwin, Ross and Tibo, Alessandro and Janet, Jon Paul and Olsson, Simon},
  journal={arXiv preprint arXiv:2406.07266},
  year={2024}
}

@article{GenerativeGradients,
  title={Generative Modeling by Estimating Gradients of the Data Distribution},
  author={Yang Song and Stefano Ermon},
  journal={ArXiv},
  year={2019},
  volume={abs/1907.05600}
}

@article{albergo2023stochastic,
  title={Stochastic interpolants: A unifying framework for flows and diffusions},
  author={Albergo, Michael S and Boffi, Nicholas M and Vanden-Eijnden, Eric},
  journal={arXiv preprint arXiv:2303.08797},
  year={2023}
}

@article{ho2022video,
  title={Video diffusion models},
  author={Ho, Jonathan and Salimans, Tim and Gritsenko, Alexey and Chan, William and Norouzi, Mohammad and Fleet, David J},
  journal={Advances in neural information processing systems},
  volume={35},
  pages={8633--8646},
  year={2022}
}

@article{singer2022make,
  title={Make-a-video: Text-to-video generation without text-video data},
  author={Singer, Uriel and Polyak, Adam and Hayes, Thomas and Yin, Xi and An, Jie and Zhang, Songyang and Hu, Qiyuan and Yang, Harry and Ashual, Oron and Gafni, Oran and others},
  journal={arXiv preprint arXiv:2209.14792},
  year={2022}
}

@article{li2022diffusionlm,
  title={Diffusion-lm improves controllable text generation},
  author={Li, Xiang and Thickstun, John and Gulrajani, Ishaan and Liang, Percy S and Hashimoto, Tatsunori B},
  journal={Advances in neural information processing systems},
  volume={35},
  pages={4328--4343},
  year={2022}
}

@ARTICLE{I2GNN,
       author = {{Huang}, Yinan and {Peng}, Xingang and {Ma}, Jianzhu and {Zhang}, Muhan},
        title = "{Boosting the Cycle Counting Power of Graph Neural Networks with I$^2$-GNNs}",
      journal = {arXiv e-prints},
         year = 2022,
        month = oct,
          eid = {arXiv:2210.13978},
        pages = {arXiv:2210.13978},
          doi = {10.48550/arXiv.2210.13978},
archivePrefix = {arXiv},
       eprint = {2210.13978},
 primaryClass = {cs.LG},
       adsurl = {https://ui.adsabs.harvard.edu/abs/2022arXiv221013978H}
}

@inproceedings{IDAwareGNN,
  title={Identity-aware Graph Neural Networks},
  author={Jiaxuan You and Jonathan M. Gomes-Selman and Rex Ying and Jure Leskovec},
  booktitle={AAAI Conference on Artificial Intelligence},
  year={2021}
}

@inproceedings{RethinkingGraphCanon,
 author = {Dong, Zehao and Zhang, Muhan and Payne, Philip and Province, Michael and Cruchaga, Carlos and Zhao, Tianyu and Li, Fuhai and Chen, Yixin},
 booktitle = {International Conference on Learning Representations},
 pages = {30797--30817},
 title = {Rethinking the Power of Graph Canonization in Graph Representation Learning with Stability},
 volume = {2024},
 year = {2024}
}

@InProceedings{li2024geometric,
  title = 	 {Geometric Representation Condition Improves Equivariant Molecule Generation},
  author =       {Li, Zian and Zhou, Cai and Wang, Xiyuan and Peng, Xingang and Zhang, Muhan},
  booktitle = 	 {Proceedings of the 42nd International Conference on Machine Learning},
  pages = 	 {36921--36953},
  year = 	 {2025},
  volume = 	 {267},
  series = 	 {Proceedings of Machine Learning Research},
  month = 	 {13--19 Jul},
  publisher =    {PMLR},
  url = 	 {https://proceedings.mlr.press/v267/li25dz.html}
}

@inproceedings{cai2023connection,
  title={On the connection between mpnn and graph transformer},
  author={Cai, Chen and Hy, Truong Son and Yu, Rose and Wang, Yusu},
  booktitle={International conference on machine learning},
  pages={3408--3430},
  year={2023},
  organization={PMLR}
}

@InProceedings{zhou2024KTransformer,
  title = 	 {On the Theoretical Expressive Power and the Design Space of Higher-Order Graph Transformers},
  author =       {Zhou, Cai and Yu, Rose and Wang, Yusu},
  booktitle = 	 {Proceedings of The 27th International Conference on Artificial Intelligence and Statistics},
  pages = 	 {2179--2187},
  year = 	 {2024},
  volume = 	 {238},
  series = 	 {Proceedings of Machine Learning Research},
  month = 	 {02--04 May},
  publisher =    {PMLR}
}

@article{PPGN,
  title={Provably Powerful Graph Networks},
  author={Haggai Maron and Heli Ben-Hamu and Hadar Serviansky and Yaron Lipman},
  journal={ArXiv},
  year={2019},
  volume={abs/1905.11136}
}

@article{GPS,
  title={Recipe for a General, Powerful, Scalable Graph Transformer},
  author={Ladislav Rampasek and Mikhail Galkin and Vijay Prakash Dwivedi and Anh Tuan Luu and Guy Wolf and D. Beaini},
  journal={ArXiv},
  year={2022},
  volume={abs/2205.12454}
}

@article{zhang2021labeling,
  title={Labeling trick: A theory of using graph neural networks for multi-node representation learning},
  author={Zhang, Muhan and Li, Pan and Xia, Yinglong and Wang, Kai and Jin, Long},
  journal={Advances in Neural Information Processing Systems},
  volume={34},
  pages={9061--9073},
  year={2021}
}

@article{gong2022diffuseq,
  title={Diffuseq: Sequence to sequence text generation with diffusion models},
  author={Gong, Shansan and Li, Mukai and Feng, Jiangtao and Wu, Zhiyong and Kong, LingPeng},
  journal={arXiv preprint arXiv:2210.08933},
  year={2022}
}

@article{watson2023rfdiffusion,
  title={De novo design of protein structure and function with RFdiffusion},
  author={Watson, Joseph L and Juergens, David and Bennett, Nathaniel R and Trippe, Brian L and Yim, Jason and Eisenach, Helen E and Ahern, Woody and Borst, Alexander J and Ragotte, Robert J and Milles, Lukas F and others},
  journal={Nature},
  volume={620},
  pages={1089--1100},
  year={2023},
  publisher={Nature Publishing Group}
}

@inproceedings{corso2023diffdock,
  title={DiffDock: Diffusion Steps, Twists, and Turns for Molecular Docking},
  author={Corso, Gabriele and St{"a}rk, Hannes and Jing, Bowen and Barzilay, Regina and Jaakkola, Tommi},
  booktitle={International Conference on Learning Representations},
  year={2023}
}

@misc{knyazev2018spectralpartitioningsignedgraphs,
      title={On spectral partitioning of signed graphs}, 
      author={Andrew V. Knyazev},
      year={2018},
      eprint={1701.01394},
      archivePrefix={arXiv},
      primaryClass={cs.DS},
      url={https://arxiv.org/abs/1701.01394}, 
}

@article{doi:10.1137/0611030,
author = {Pothen, Alex and Simon, Horst D. and Liou, Kang-Pu},
title = {Partitioning Sparse Matrices with Eigenvectors of Graphs},
journal = {SIAM Journal on Matrix Analysis and Applications},
volume = {11},
number = {3},
pages = {430-452},
year = {1990},
doi = {10.1137/0611030}
}

@misc{luo2021graphdfdiscreteflowmodel,
      title={GraphDF: A Discrete Flow Model for Molecular Graph Generation}, 
      author={Youzhi Luo and Keqiang Yan and Shuiwang Ji},
      year={2021},
      eprint={2102.01189},
      archivePrefix={arXiv},
      primaryClass={cs.LG},
      url={https://arxiv.org/abs/2102.01189}, 
}

@misc{jin2019junctiontreevariationalautoencoder,
      title={Junction Tree Variational Autoencoder for Molecular Graph Generation}, 
      author={Wengong Jin and Regina Barzilay and Tommi Jaakkola},
      year={2019},
      eprint={1802.04364},
      archivePrefix={arXiv},
      primaryClass={cs.LG},
      url={https://arxiv.org/abs/1802.04364}, 
}

@misc{simonovsky2018graphvaegenerationsmallgraphs,
      title={GraphVAE: Towards Generation of Small Graphs Using Variational Autoencoders}, 
      author={Martin Simonovsky and Nikos Komodakis},
      year={2018},
      eprint={1802.03480},
      archivePrefix={arXiv},
      primaryClass={cs.LG},
      url={https://arxiv.org/abs/1802.03480}, 
}

@misc{shi2020graphafflowbasedautoregressivemodel,
      title={GraphAF: a Flow-based Autoregressive Model for Molecular Graph Generation}, 
      author={Chence Shi and Minkai Xu and Zhaocheng Zhu and Weinan Zhang and Ming Zhang and Jian Tang},
      year={2020},
      eprint={2001.09382},
      archivePrefix={arXiv},
      primaryClass={cs.LG},
      url={https://arxiv.org/abs/2001.09382}, 
}

@misc{hou2024improvingmoleculargraphgeneration,
      title={Improving Molecular Graph Generation with Flow Matching and Optimal Transport}, 
      author={Xiaoyang Hou and Tian Zhu and Milong Ren and Dongbo Bu and Xin Gao and Chunming Zhang and Shiwei Sun},
      year={2024},
      eprint={2411.05676},
      archivePrefix={arXiv},
      primaryClass={cs.LG},
      url={https://arxiv.org/abs/2411.05676}, 
}

@misc{lee2025fragfmhierarchicalframeworkefficient,
      title={FragFM: Hierarchical Framework for Efficient Molecule Generation via Fragment-Level Discrete Flow Matching}, 
      author={Joongwon Lee and Seonghwan Kim and Seokhyun Moon and Hyunwoo Kim and Woo Youn Kim},
      year={2025},
      eprint={2502.15805},
      archivePrefix={arXiv},
      primaryClass={cs.LG},
      url={https://arxiv.org/abs/2502.15805}, 
}

@misc{feng2023generation3dmoleculespockets,
      title={Generation of 3D Molecules in Pockets via Language Model}, 
      author={Wei Feng and Lvwei Wang and Zaiyun Lin and Yanhao Zhu and Han Wang and Jianqiang Dong and Rong Bai and Huting Wang and Jielong Zhou and Wei Peng and Bo Huang and Wenbiao Zhou},
      year={2023},
      eprint={2305.10133},
      archivePrefix={arXiv},
      primaryClass={cs.LG},
      url={https://arxiv.org/abs/2305.10133}, 
}

@article{
doi:10.26434/chemrxiv-2024-0ckgt-v2,
author = {Jike Wang  and Hao Luo  and Rui Qin  and Mingyang Wang  and Meijing Fang  and Odin Zhang  and Qiaolin Gou  and Qun Su  and Chao Shen  and Ziyi You  and Xiaozhe Wan  and Liwei Liu  and Chang-Yu Hsieh  and Tingjun Hou  and Yu Kang },
title = {3D Molecular Pocket-based Generation with Token-only Large Language Model},
journal = {ChemRxiv},
volume = {2024},
number = {0819},
pages = {},
year = {2024},
doi = {10.26434/chemrxiv-2024-0ckgt-v2},
URL = {https://chemrxiv.org/doi/abs/10.26434/chemrxiv-2024-0ckgt-v2},
eprint = {https://chemrxiv.org/doi/pdf/10.26434/chemrxiv-2024-0ckgt-v2}}

@misc{hassan2024etflowequivariantflowmatchingmolecular,
      title={ET-Flow: Equivariant Flow-Matching for Molecular Conformer Generation}, 
      author={Majdi Hassan and Nikhil Shenoy and Jungyoon Lee and Hannes Stark and Stephan Thaler and Dominique Beaini},
      year={2024},
      eprint={2410.22388},
      archivePrefix={arXiv},
      primaryClass={q-bio.QM},
      url={https://arxiv.org/abs/2410.22388}, 
}

@misc{hong2025accelerating3dmoleculegeneration,
      title={Accelerating 3D Molecule Generation via Jointly Geometric Optimal Transport}, 
      author={Haokai Hong and Wanyu Lin and Kay Chen Tan},
      year={2025},
      eprint={2405.15252},
      archivePrefix={arXiv},
      primaryClass={cs.LG},
      url={https://arxiv.org/abs/2405.15252}, 
}

@misc{huang2023learningjoint2d,
      title={Learning Joint 2D \& 3D Diffusion Models for Complete Molecule Generation}, 
      author={Han Huang and Leilei Sun and Bowen Du and Weifeng Lv},
      year={2023},
      eprint={2305.12347},
      archivePrefix={arXiv},
      primaryClass={q-bio.BM},
      url={https://arxiv.org/abs/2305.12347}, 
}

@misc{hua2024mudiffunifieddiffusioncomplete,
      title={MUDiff: Unified Diffusion for Complete Molecule Generation}, 
      author={Chenqing Hua and Sitao Luan and Minkai Xu and Rex Ying and Jie Fu and Stefano Ermon and Doina Precup},
      year={2024},
      eprint={2304.14621},
      archivePrefix={arXiv},
      primaryClass={cs.LG},
      url={https://arxiv.org/abs/2304.14621}, 
}

@misc{tong2024improvinggeneralizingflowbasedgenerative,
      title={Improving and generalizing flow-based generative models with minibatch optimal transport}, 
      author={Alexander Tong and Kilian Fatras and Nikolay Malkin and Guillaume Huguet and Yanlei Zhang and Jarrid Rector-Brooks and Guy Wolf and Yoshua Bengio},
      year={2024},
      eprint={2302.00482},
      archivePrefix={arXiv},
      primaryClass={cs.LG},
      url={https://arxiv.org/abs/2302.00482}, 
}

@misc{decao2022molganimplicitgenerativemodel,
      title={MolGAN: An implicit generative model for small molecular graphs}, 
      author={Nicola De Cao and Thomas Kipf},
      year={2022},
      eprint={1805.11973},
      archivePrefix={arXiv},
      primaryClass={stat.ML},
      url={https://arxiv.org/abs/1805.11973}, 
}

@misc{jiang2025bureswassersteinflowmatchinggraph,
      title={Bures-Wasserstein Flow Matching for Graph Generation}, 
      author={Keyue Jiang and Jiahao Cui and Xiaowen Dong and Laura Toni},
      year={2025},
      eprint={2506.14020},
      archivePrefix={arXiv},
      primaryClass={cs.LG},
      url={https://arxiv.org/abs/2506.14020}, 
}

@misc{tian2024equiflowequivariantconditionalflow,
      title={EquiFlow: Equivariant Conditional Flow Matching with Optimal Transport for 3D Molecular Conformation Prediction}, 
      author={Qingwen Tian and Yuxin Xu and Yixuan Yang and Zhen Wang and Ziqi Liu and Pengju Yan and Xiaolin Li},
      year={2024},
      eprint={2412.11082},
      archivePrefix={arXiv},
      primaryClass={cs.LG},
      url={https://arxiv.org/abs/2412.11082}, 
}

@misc{kornilov2024optimalflowmatchinglearning,
      title={Optimal Flow Matching: Learning Straight Trajectories in Just One Step}, 
      author={Nikita Kornilov and Petr Mokrov and Alexander Gasnikov and Alexander Korotin},
      year={2024},
      eprint={2403.13117},
      archivePrefix={arXiv},
      primaryClass={stat.ML},
      url={https://arxiv.org/abs/2403.13117}, 
}

@inproceedings{
quotientspace,
title={Quotient-Space Diffusion Model},
author={Anonymous},
booktitle={The Fourteenth International Conference on Learning Representations},
year={2026},
url={https://openreview.net/forum?id=3JPAkwSVc4}
}

@article{lawrence2025improving,
  title={Improving equivariant networks with probabilistic symmetry breaking},
  author={Lawrence, Hannah and Portilheiro, Vasco and Zhang, Yan and Kaba, S{\'e}kou-Oumar},
  journal={arXiv preprint arXiv:2503.21985},
  year={2025}
}
\bibliographystyle{icml2026}

\newpage
\appendix
\onecolumn
\section*{Appendix}

\section{Related Work}

\paragraph{Symmetries in generative models.} Many datasets of interest in generative modeling are naturally defined only up to group actions—e.g., rotations and translations in 3D, or permutations in sets and graphs, so a principled generator should respect these symmetries to avoid spurious modes and overcounting equivalent configurations. A dominant approach is to build symmetry into the architecture by enforcing group equivariance (and thus invariance of the induced distribution), as in E(n)/SE(3)-equivariant message passing networks and their use in equivariant flows and diffusion/score-based models for geometric data~\citep{satorras2021n, garcia2021n, hoogeboom2022equivariant}. Systematic studies further clarify the design space and practical trade-offs of equivariant generative modeling under such geometric symmetries~\cite{le2023navigating,lawrence2025improving}. An alternative architecture-agnostic perspective achieves symmetry by operating on the group orbit structure, for instance via group averaging or canonicalization/quotient representations, which can enable the use of expressive non-equivariant backbones while still yielding equivariant or invariant models; learned canonicalization has been shown to provide a general mechanism for constructing equivariant functions from canonical representatives~\cite{kaba2023canonical}. However, canonicalization is not “free”: for common groups there are fundamental continuity/stability obstructions, motivating softened or averaged constructions such as weighted frames to mitigate discontinuities near symmetric configurations~\cite{dym2024equivariant}. Related theory analyzes when canonicalization or group averaging offers better statistical behavior, highlighting distinct generalization regimes~\cite{tahmasebi2025generalization}. Complementing these perspectives, recent work in graph diffusion argues that enforcing strict invariance can make learning harder by inducing mixture-like objectives over symmetry transformations, and proposes post-hoc group randomization at sampling time to recover invariance without constraining the training model \cite{yan2023swingnn}. Finally, concurrent directions explicitly formulate diffusion on quotient spaces to avoid learning redundant degrees of freedom introduced by symmetries, further underscoring the value of orbit-space viewpoints for efficient generative learning\cite{quotientspace}.

\paragraph{Optimal transport.} In continuous-time generative modeling, optimal transport (OT) is often used as a principle for selecting cost-effective and more structured probability paths between a simple base distribution and the data distribution. A widely used choice in flow matching is the OT displacement interpolation, which yields trajectories closer to minimal-cost transport and empirically leads to faster convergence and fewer sampling steps~\citep{lipman2022flow}. Related ``trajectory straightening'' viewpoints, such as Rectified Flow, can be interpreted as progressively transforming the learned dynamics toward straighter transport-like paths, further improving numerical stability and reducing the number of integration steps required at inference~\citep{liu2022flow,tong2024improvinggeneralizingflowbasedgenerative,kornilov2024optimalflowmatchinglearning}.

In molecular generation, OT-inspired objectives have been adopted more explicitly to enable high-quality few-step generation under geometric symmetries. For 2D molecular graph generation, MolGAN~\citep{decao2022molganimplicitgenerativemodel} was among the first to successfully use the Wasserstein-1 distance to stabilize the training of molecular graph generators. Subsequent work such as BWFlow~\citep{jiang2025bureswassersteinflowmatchinggraph} also demonstrates the effectiveness of OT in 2D molecule design. For 3D molecule generation, various methods have demonstrated that OT can serve not only as a theoretical lens but also as a practical design tool for fast, high-fidelity molecule synthesis~\citep{song2024equivariant, tian2024equiflowequivariantconditionalflow,hong2025accelerating3dmoleculegeneration}.

\paragraph{Molecular generative models.}

Early work largely focused on generating discrete molecular structures via string or graph-based parameterizations, including structured latent-variable models that explicitly construct chemically valid graphs~\citep{jin2019junctiontreevariationalautoencoder,simonovsky2018graphvaegenerationsmallgraphs, jang2023hierarchical}. A major recent line uses diffusion processes on discrete graph attributes. Discrete denoising diffusion models operate directly on categorical attributes and have demonstrated strong performance at scale on molecular benchmarks~\citep{vignac2022digress}. In parallel, flow-based methods provide alternative likelihood-based or transport-inspired formulations for molecular graph generation~\citep{luo2021graphdfdiscreteflowmodel,shi2020graphafflowbasedautoregressivemodel}, and discrete flow matching further improves sampling flexibility and efficiency while retaining strong generation quality~\citep{qin2024defog,hou2024improvingmoleculargraphgeneration,lee2025fragfmhierarchicalframeworkefficient,chen2023efficient,luo2023fast}.

Generating molecules directly in 3D space has been explored with sequential and autoregressive schemes that place atoms step-by-step while maintaining geometric consistency~\citep{gebauer2019symmetry, luo2022autoregressive,feng2023generation3dmoleculespockets,doi:10.26434/chemrxiv-2024-0ckgt-v2}. Diffusion models have become a dominant paradigm for 3D molecular generation by learning to denoise corrupted coordinates with architectures designed to respect Euclidean symmetries~\citep{hoogeboom2022equivariant}. Closely related conditional settings, such as conformer generation given a fixed molecular graph, have also benefited from geometric diffusion and flow matching formulations~\citep{xu2022geodiff,hassan2024etflowequivariantflowmatchingmolecular,hong2025accelerating3dmoleculegeneration, wu2022diffusion}. 

A growing body of work aims to jointly generate discrete molecular graphs and continuous 3D geometries to avoid brittle post-hoc bond inference and to better couple chemical structure with spatial arrangement. MiDi~\citep{vignac2023midi} proposes a mixed discrete--continuous diffusion approach that generates molecular graphs together with conformers in an end-to-end differentiable manner. Subsequent efforts further explore joint diffusion over comprehensive molecular representations~\citep{huang2023learningjoint2d,hua2024mudiffunifieddiffusioncomplete,irwin2024semlaflow}.

There are also latent generative models developed for 2D, 3D, or joint representations of molecules, including GeoLDM~\citep{xu2023geometric}, LGD~\citep{latentgraphdiffusion}, and LDM-3DG~\citep{you2023latent}. More recently, GeoRCG~\citep{li2024geometric} further utilizes a two-stage generation: first generating molecule representations, then using geometric representations to guide the molecule generation. We emphasize that our analysis and proposed methods are also applicable and effective for these latent/representation space generative models under mild conditions.

\section{Omitted Proof}\label{sec_appendix:proof}

This section provides all omitted proof and more detailed discussions on our theoretical results, presented mainly in \Cref{sec:theory}.

\subsection{Proof for \Cref{subsec:correct_superior}}\label{subsec_appendix:proof_correct}
This subsection proves that canonicalized generative models can induce universal invariant distributions and outperform those equivariant baselines trained directly on invariant data.
\subsubsection{Proof for \Cref{subsubsec:correct}}
We first show the correctness and sufficiency of canonical slice modeling. 

\begin{theorem}[Factorization of invariant measures; \Cref{thm:factor} in main text]
Suppose Assumptions~\ref{ass:free} and \ref{ass:center} hold. Let $\mu$ be any $\G$-invariant probability measure on $\M$. Let $\Psi$ be an orbit representative map defined $\mu$-a.s., and let $\nu=\Psi\push\mu$ be the slice distribution on $S=\Psi(\M)$. Then
\begin{equation}
\mu \;=\; \int_{S}\left(\int_{\G}\delta_{g\acts \Z}\,\mathrm{d}\Haar(g)\right)\mathrm{d}\nu(\Z).
\end{equation}
Equivalently, if $\tilde\Z\sim \nu$, $g\sim\Haar$ independent, then $g\acts \tilde\Z\sim \mu$.
\end{theorem}

\begin{proof}
This is a known result, included for completeness.
Let $\Z\sim \mu$ and define $\tilde\Z=\Psi(\Z)\in S$. By definition of orbit representative, there exists $g(\Z)\in\G$ with $\Z=g(\Z)\acts \tilde\Z$. Under \ref{ass:free}, $g(\Z)$ is unique a.s. Fix any measurable $A\subseteq \G$. For any $h\in\G$, invariance of $\mu$ implies
\[
\Prob(g(\Z)\in A\mid \tilde\Z) = \Prob(h\,g(\Z)\in A\mid \tilde\Z),
\]
so the conditional law of $g(\Z)$ given $\tilde\Z$ is left-invariant. By uniqueness of Haar probability on compact $\G$, this conditional law is $\Haar$. Therefore,
\begin{equation}
\begin{aligned}
\mu(B)&=\E[\1\{g(\Z)\acts \tilde\Z\in B\}]
=\E\Big[\int_{\G}\1\{g\acts \tilde\Z\in B\}\,\mathrm{d}\Haar(g)\Big]\\
&=\int_S \Big(\int_\G \delta_{g\acts \Z}(B)\,\mathrm{d}\Haar(g)\Big)\mathrm{d}\nu(\Z),
\end{aligned}
\end{equation}
which is \eqref{eq:factor}. \qedhere
\end{proof}

\begin{corollary}[Sufficiency of slice modeling; Corollary~\ref{cor:suff} in main text]
To model any invariant target $\mu$, it suffices to model the slice distribution $\nu$; invariance is recovered by Haar randomization.
\end{corollary}

In the next part, we show the universality of canonicalized parameterizations over invariant and equivariant functions.

\begin{proposition}[Universality over equivariant parameterizations; Proposition~\ref{thm:universal_inv_equi} in the main text]
Let $\G$ act continuously on a compact set $K\subset \R^d$. Suppose we have a (measurable) canonicalization map $\Psi:K\to K$ as \Cref{def:map_slice}, and a (measurable) gauge map $\kappa:K\to \G$ s.t.
\begin{equation}
\Psi(g\acts x)=\Psi(x),\ \kappa(g\acts x)=g\kappa(x),\ x=\kappa(x)\acts \Psi(x)
\end{equation}
(defined on a full-measure set; under free actions, $\kappa$ is unique a.s.). Consider the parametrization
\begin{equation}
\phi(x)=\kappa(x)\acts f(\Psi(x)),
\end{equation}
where $f$ is a universal approximator on $\Psi(K)$. Then $\phi$ is a universal approximator of continuous $\G$-equivariant functions on $K$, and $f\circ \Psi$ is universal for continuous $\G$-invariant functions on $K$.
\end{proposition}

\begin{proof}
Let $F:K\to \R^d$ be continuous and $\G$-equivariant: $F(g\acts x)=g\acts F(x)$. Define its restriction to the slice $S:=\Psi(K)$ by $F_S(z):=F(z)$ for $z\in S$.
For any $x\in K$, write $x=\kappa(x)\acts \Psi(x)$. By equivariance,
\begin{equation}
F(x)=F(\kappa(x)\acts \Psi(x))=\kappa(x)\acts F(\Psi(x))=\kappa(x)\acts F_S(\Psi(x)).
\end{equation}
Since $f$ is universal on $S$, for any $\varepsilon>0$ there exists $f$ such that $\sup_{z\in S}\norm{f(z)-F_S(z)}<\varepsilon$. Then
\begin{equation}
\sup_{x\in K}\norm{\phi(x)-F(x)}\le \sup_{z\in S}\norm{f(z)-F_S(z)}<\varepsilon,
\end{equation}
using that the group action is norm-preserving.
\end{proof}

\subsubsection{Proof for \Cref{subsubsec:superior}}\label{subsubsec_appendix:proof_superior}
We now show the superiority of canonicalized generative models with expressive (possibly) non-equivariant networks over standard generative models with equivariant backbones. 

\paragraph{Expressivity of equivariant and non-equivariant models.} It is known that the expressivity of message-passing based GNNs are bounded by the first-order Weisfeiler–Lehman (1-WL) algorithm on abstract graph in terms of distinguishing isomorphic graphs~\citep{HowPowerfulGNN}. Following immediately is their inability in counting substructures such as cycles~\citep{I2GNN} and even link prediction~\citep{zhang2021labeling}. Analogous issues exist for transformers due to their equivalence with MPNNs~\citep{cai2023connection}, and equivariant GNNs including EGNNs~\citep{satorras2021n} when processing 3D geometric graphs~\citep{li2024distance,li2024completeness}. There are some main stream methods to overcome the expressivity pitfalls:
\begin{enumerate}
    \item \emph{High-order architectures}: high-order GNNs~\citep{morris2019weisfeiler,morris2020weisfeiler,PPGN} and high-order transformers~\citep{Kim2023pure,zhou2024KTransformer} have provably more powerful expressivity within and beyond the $k$-WL framework, at the cost of exponential more computation costs; for example, SwinGNN~\citep{yan2023swingnn} adopts $3$-WL equivalent PPGN~\citep{PPGN} as the denoising backbone. We argue, however, stronger expressivity does not need to emerge from high-order networks with exponential complexity: symmetry breaking also improves expressivity as detailed as follows.
    \item \emph{Symmetry breaking}: by introducing asymmetry intentionally, the expressivity of underlying network can be provably enhanced. A widely adopted technique is ``labeling" the nodes with extra identity features~\citep{IDAwareGNN,zhang2021labeling,I2GNN,klWL}, and the expressivity monotonically increases with the number of IDs. Unfortunately, these methods need the relational pooling or averaging to produce invariant outputs, which again introduce complexity at the scale of order of the group - canonicalization offers a way to produce appropriate and stable outputs~\citep{RethinkingGraphCanon} without the costly traverse and averaging procedure.
    \item \emph{Positional encodings and structural encodings}: graph structure-based PEs significantly improve the expressivity and practical performance of neural backbones~\citep{GPS}. Popular PEs include spatial structure-based (e.g. RWSE~\citep{RWSELearnable} and RRWP~\citep{Exphormer}) and spectra-based (using eigenvalues and eigenvectors of graph Laplacians~\citep{LaplacianCanonization,SignBasisNet,Specformer} or Hodge Laplacians~\citep{zhou2023hodge}). However, different from the fixed positional encodings in images or texts based on absolute indices of tokens, the calculations of above PEs on graphs rely on ground truth structures, which is not available in diffusion generation. Some generative models like DeFoG~\citep{qin2024defog} utilizes an estimated PE from the currect noisy graph, which we argue might not be the optimal choice. In comparison, our canonicalization rank can also be viewed as spatial and spectral PEs, yet yields almost no computation overhead and no train-test gap.
\end{enumerate}

\paragraph{Provable smaller TV error with canonicalized non-equivariant backbones.}
Following immediately is the superior expressivity of combining non-equivariant denoising networks with canonicalization.

\begin{assumption}[Generative modeling error of equivariant models]\label{ass:univ}
For any target invariant distribution $\mu$, suppose there exist parameters $\theta(\text{eq})$ such that the equivariant generative model $\hat\nu_{\theta( \text{equi})}$ satisfies $\TV(\hat\mu_{\theta(\text{eq})},\mu)<\varepsilon$.
\end{assumption}
We hence conclude that the generative modeling error could be smaller for canonicalized non-equivariant models.

\begin{theorem}[Canonicalized models match (or improve) equivariant baselines in TV]\label{thm:univ-inv}
Under Assumptions \ref{ass:free} and \ref{ass:center}, let $\mu$ be any $\G$-invariant target distribution on $\M$ and let $\nu:=\Psi_{\#}\mu$ be its canonical-slice distribution on $S=\Psi(\M)$.
Assume there exists an \emph{equivariant baseline} (defined in ambient space) producing an invariant model distribution $\hat\mu_{\theta(\mathrm{eq})}$ such that
\begin{equation}
\TV\big(\hat\mu_{\theta(\mathrm{eq})},\mu\big)<\varepsilon.
\end{equation}
Let $\hat\nu_{\theta(\mathrm{eq})}:=\Psi_{\#}\hat\mu_{\theta(\mathrm{eq})}$ be the induced slice distribution.
If the slice model class $\{\hat\nu_{\theta(\mathrm{free})}\}$ is at least as expressive on $S$ in the sense that
\begin{equation}\label{eq:ass-slice-dominates}
\inf_{\theta(\mathrm{free})}\TV\big(\hat\nu_{\theta(\mathrm{free})},\nu\big)
\;\le\;
\TV\big(\hat\nu_{\theta(\mathrm{eq})},\nu\big),
\end{equation}
then there exists $\theta(\mathrm{free})$ such that the Haar-randomized model
\begin{equation}
\hat\mu_{\theta(\mathrm{free})}
:=\int_{S}\left(\int_{\G}\delta_{g\acts \Z}\,\mathrm{d}\Haar(g)\right)\mathrm{d}\hat\nu_{\theta(\mathrm{free})}(\Z)
\end{equation}
satisfies $\TV\big(\hat\mu_{\theta(\mathrm{free})},\mu\big)<\varepsilon$.
\end{theorem}

\begin{proof}
First, pushforward cannot increase TV (data processing inequality). Hence
\begin{equation}
\TV(\hat\nu_{\theta(\mathrm{eq})},\nu)
=\TV\big(\Psi_{\#}\hat\mu_{\theta(\mathrm{eq})},\Psi_{\#}\mu\big)
\le \TV\big(\hat\mu_{\theta(\mathrm{eq})},\mu\big)
<\varepsilon.
\end{equation}
By \eqref{eq:ass-slice-dominates} (note that this usually holds according to our discussions above), choose $\theta(\mathrm{free})$ such that $\TV(\hat\nu_{\theta(\mathrm{free})},\nu)\le\TV(\hat\nu_{\theta(\mathrm{eq})},\nu)<\varepsilon$.
Finally, Haar randomization is a Markov kernel $T$ (cf.\ Theorem~\ref{thm:factor}), and thus is TV-nonexpansive:
\begin{equation}
\TV\big(T(\hat\nu_{\theta(\mathrm{free})}),T(\nu)\big)\le \TV\big(\hat\nu_{\theta(\mathrm{free})},\nu\big).
\end{equation}
Since $T(\nu)=\mu$ (Theorem~\ref{thm:factor}), we conclude $\TV(\hat\mu_{\theta(\mathrm{free})},\mu)<\varepsilon$. \qedhere
\end{proof}

\subsection{Proof for \Cref{subsec:accelerate}}\label{subsec_appendix:proof_accelerate}
In this subsection, we provide more detailed analysis and complete proof for our results regarding accelerating diffusion and flow model training through canonicalization.

\subsubsection{Proof for \Cref{subsubsec:score}}
This part mainly considers score-based diffusion modeling, showing that the existence of group symmetry in invariant data induces the mixture structure of the score function.
\begin{proposition}[Score of a symmetry mixture]\label{prop:mixture-score}
Assume $q$ is differentiable and positive where needed. Then
\begin{equation}
\grad \log p(x)=\sum_{m=1}^M w_m(x)\, g_m\acts \grad \log q(g_m^{-1}\acts x),
\quad
w_m(x):=\frac{q(g_m^{-1}\acts x)}{\sum_{j=1}^M q(g_j^{-1}\acts x)}.
\end{equation}
\end{proposition}
\begin{proof}
Differentiate $\log p(x)=\log\left(\frac{1}{M}\sum_m q(g_m^{-1}\acts x)\right)$ and apply the chain rule. Orthogonality implies $\grad_x q(g_m^{-1}\acts x)=g_m\acts \grad q(g_m^{-1}\acts x)$. \qedhere
\end{proof}

\subsubsection{Proof for \Cref{subsubsec:variance}}\label{subsubsec_appendix:proof_variance}
This part highlights our central result: canonical training reduces conditional variance in flow-matching. 

\paragraph{Irreducible error in flow matching.} First we prove a standard lemma, stating that the irreducible flow-matching error is the conditional variance.
\begin{lemma}[Irreducible flow-matching error = conditional variance]\label{lem:variance}
Let $\hat v$ be any measurable predictor of $v_t(\Z_t)$ from $(t,\Z_t)$. Then the minimum achievable MSE satisfies
\begin{equation}
\inf_{\hat v}\E\big[\norm{\hat v(t,\Z_t)-(\Z_1-\Z_0)}^2\big]
=\E\big[\mathrm{Var}(\Z_1-\Z_0\mid t,\Z_t)\big].
\end{equation}
\end{lemma}
\begin{proof}
This is the standard $L^2$ regression identity: the conditional expectation $\E[\Z_1-\Z_0\mid t,\Z_t]$ is the unique minimizer, and the minimum risk equals the expected conditional variance (law of total variance). \qedhere
\end{proof}

\begin{corollary}[Deterministic couplings yield zero irreducible error]\label{cor:det}
If $\Z_1-\Z_0$ is (a.s.) a deterministic function of $(t,\Z_t)$, then the right-hand side is $0$, and in principle the flow can be learned with arbitrarily small error (up to approximation and optimization).
\end{corollary}

Below we give a concrete example of this irreducible error, showing that the minimum achievable MSE loss is strictly positive at the presence of multi-group ambiguity.

\paragraph{Quantifying symmetry ambiguity through posterior collision bounds}\label{sec:quantify-ambiguity-and-misalignment}
We now strengthens the variance-decomposition viewpoint by (i) giving an exact identity for the
\emph{symmetry-ambiguity} term in \Cref{thm:canon-var-decomp},
(ii) deriving lower bounds purely in terms of the posterior $\varrho(g\mid Z_t)$ over the latent symmetry.

Under the group-aligned lifted coupling, write the posterior over the latent symmetry as $\varrho(g\mid z):=\mathbb{P}(G=g\mid Z_t=z)$ (finite group case;
for compact groups interpret $\varrho(\cdot\mid z)$ as a posterior density w.r.t.\ Haar).

We first show that ambiguity term can be written as a pairwise mean-separation identity.
The following identity is standard (\emph{variance via i.i.d.\ copies}), but its specialization here yields an explicit
geometric handle on the symmetry-ambiguity term.

\begin{lemma}[Symmetry ambiguity as pairwise conditional-mean separation]
\label{lem:ambiguity-pairwise}
Define the group-conditioned conditional mean drift
\begin{equation}
m_g(z) := \mathbb{E}[U\mid Z_t=z,\;G=g],
\qquad
m(z) := \mathbb{E}[U\mid Z_t=z].
\end{equation}
Then the ambiguity term in \Cref{thm:canon-var-decomp} satisfies, for each $z$,
\begin{align}
\mathrm{Var}\!\left(\mathbb{E}[U\mid Z_t,G]\mid Z_t=z\right)
&= \mathbb{E}_{g\sim \varrho(\cdot\mid z)}\big[\|m_g(z)-m(z)\|^2\big] \label{eq:ambiguity-mean-var}\\
&= \frac{1}{2}\;\mathbb{E}_{g,g'\,\overset{\text{i.i.d.}}{\sim}\, \varrho(\cdot\mid z)}\big[\|m_g(z)-m_{g'}(z)\|^2\big]. \label{eq:ambiguity-pairwise}
\end{align}
\end{lemma}

\begin{proof}
Eq.~\eqref{eq:ambiguity-mean-var} is the definition of conditional variance of a discrete random variable taking values
$m_G(z)$. Eq.~\eqref{eq:ambiguity-pairwise} follows from the standard identity
$\mathrm{Var}(X)=\frac{1}{2}\mathbb{E}\|X-X'\|^2$ for i.i.d.\ copies $(X,X')$ (applied conditionally on $Z_t=z$).
\end{proof}

We hereby provide the lower bounds from posterior uncertainty (collision probability).
Lemma~\ref{lem:ambiguity-pairwise} implies that ambiguity is large whenever the posterior $\varrho(g\mid Z_t)$ spreads mass
over group elements with meaningfully different conditional mean drifts.

\begin{proposition}[Posterior-collision lower bound]
\label{prop:collision-lb}
Fix $z$ and let $\mathcal{G}_\Delta(z)\subseteq\mathcal{G}$ be any subset such that for all $g\in\mathcal{G}_\Delta(z)$,
\begin{equation}
\|m_g(z)-m(z)\|\ge \Delta(z).
\end{equation}
Then
\begin{equation}
\mathrm{Var}\!\left(\mathbb{E}[U\mid Z_t,G]\mid Z_t=z\right)
\;\ge\;
\Delta(z)^2\;\mathbb{P}\big(G\in \mathcal{G}_\Delta(z)\mid Z_t=z\big).
\end{equation}
In particular, if there exist two disjoint subsets $\mathcal{A}(z),\mathcal{B}(z)$ such that
$\|m_g(z)-m_{g'}(z)\|\ge \Delta(z)$ for all $g\in\mathcal{A}(z)$ and $g'\in\mathcal{B}(z)$, then
\begin{equation}
\mathrm{Var}\!\left(\mathbb{E}[U\mid Z_t,G]\mid Z_t=z\right)
\;\ge\;
\frac{\Delta(z)^2}{2}\;
\Big(1-\sum_{g} \varrho(g\mid z)^2\Big),
\end{equation}
where $1-\sum_g \varrho(g\mid z)^2$ is the complement of the posterior collision probability.
\end{proposition}

\begin{proof}
The first bound is immediate from \eqref{eq:ambiguity-mean-var} by keeping only $g\in\mathcal{G}_\Delta(z)$ and using
$\|m_g(z)-m(z)\|^2\ge \Delta(z)^2$ there.
For the second bound, use \eqref{eq:ambiguity-pairwise} and lower bound the pairwise distance by $\Delta(z)$ whenever
$(g,g')\in\mathcal{A}(z)\times\mathcal{B}(z)\cup \mathcal{B}(z)\times\mathcal{A}(z)$. Then
\begin{equation}
\mathbb{P}(g\neq g'\mid z)=1-\sum_g \varrho(g\mid z)^2,
\end{equation}
and the stated inequality follows (up to the factor $1/2$ from \eqref{eq:ambiguity-pairwise}).
\end{proof}

The collision term $1-\sum_g \varrho(g\mid Z_t)^2$ is $0$ iff $G$ is a.s.\ determined by $Z_t$.
Thus, unless the conditional mean drifts $m_g(Z_t)$ coincide across $g$, symmetry ambiguity creates a strictly positive
variance component whenever $G$ remains uncertain.

\paragraph{Discussions on lifted coupling.}
A subtle but important point is that the group-aligned lift in~\eqref{eq:group-aligned-lift} shares the same
latent $G$ across $(Z_0,Z_1)$, which \emph{a priori} could introduce dependence.
However, for \emph{standard isotropic Gaussian} priors and orthogonal group actions (including
$S_N$ permutations and $SO(3)$ rotations), this dependence disappears.

\begin{proposition}[Lift of a product slice coupling equals the ambient product coupling]
\label{prop:lift-product}
Assume $\mathcal G$ acts orthogonally on $\mathbb R^d$.
Let $G\sim \lambda$ and $(\widetilde Z_0,\widetilde Z_1)\sim q_0\otimes \mathcal N(0,I)$ be independent, and define
$(Z_0,Z_1)=(G\cdot \widetilde Z_0,\;G\cdot \widetilde Z_1)$.
Then (i) $Z_1\sim \mathcal N(0,I)$ and $Z_1\perp G$, and (ii) $Z_1\perp Z_0$.
Moreover, if the invariant data distribution satisfies the disintegration
$p_0 = \int (g\cdot)_{\#} q_0\, d\lambda(g)$, then the lifted ambient coupling satisfies
\begin{equation}
(Z_0,Z_1)\sim p_0\otimes \mathcal N(0,I).
\end{equation}
\end{proposition}

\begin{proof}
For any measurable $B$ and any $g$, orthogonal invariance of $\mathcal N(0,I)$ gives
$\mathbb P(Z_1\in B\mid G=g)=\mathbb P(g\cdot \widetilde Z_1\in B)=\mathbb P(\widetilde Z_1\in B)$,
which does not depend on $g$, hence $Z_1\perp G$ and $Z_1\sim \mathcal N(0,I)$.
For any measurable $A,B$,
\begin{equation}
\mathbb P(Z_0\in A, Z_1\in B)
=\mathbb E\!\left[\mathbf 1_{\{Z_0\in A\}}\mathbb P(Z_1\in B\mid G,\widetilde Z_0)\right]
=\mathbb P(\widetilde Z_1\in B)\,\mathbb P(Z_0\in A),
\end{equation}
since $\widetilde Z_1$ is independent of $(G,\widetilde Z_0)$ and $\mathbb P(g\cdot \widetilde Z_1\in B)$ is constant in $g$.
Thus $Z_1\perp Z_0$. If additionally $Z_0\sim p_0$ via the stated disintegration, the joint law is
$p_0\otimes\mathcal N(0,I)$.
\end{proof}
Therefore, the commonly used ambient product coupling $\gamma_{\mathrm{mix}}=p_0\otimes\mathcal N(0,I)$ can be viewed as a special case of slice training with product coupling and isotropic Gaussian slice prior, thus \emph{conditional variances of two paradigms become directly comparable}.
Canonicalization does \emph{not} change the coupling in this case; it changes the \emph{representation} so that the group ambiguity term in \Cref{thm:canon-var-decomp} becomes explicit and avoidable on the slice. Hence our theory covers most practical training regimes without optimal transport.

However, this generally does not hold if $\tilde Z_1$ is group-aware (e.g., learned or group-statistics related prior in our implementation), or $\tilde Z_1$ and $\tilde Z_0$ are not independent (e.g., optimal transport on the canonicalized slice). In particular, if the slice prior is \emph{not} isotropic (e.g., an aligned Gaussian $\mathcal N(\mu,\Sigma)$ with $g\Sigma g^\top\neq \Sigma$ for some $g$), then $g\cdot \widetilde Z_1$ depends on $g$ in distribution, so $Z_1$ is no longer independent of $G$, and the lifted coupling generally is \emph{not} a product.
Likewise, if the slice coupling $\widetilde\gamma$ is non-product (e.g., OT/Monge), the lifted coupling remains non-product.

\paragraph{Canonicalization reduces conditional variance via symmetry ambiguity elimination.} We now prove our central result on the advantages of canonicalized flow matching.
\begin{theorem}[Variance decomposition under group-aligned lift; \Cref{thm:canon-var-decomp} in main text]
Assume $\mathcal{G}$ acts orthogonally. Under the group-aligned lifted coupling
$G\sim\lambda$ independent of $(\widetilde Z_0,\widetilde Z_1)$ and $(Z_0,Z_1)=(G\cdot\widetilde Z_0,G\cdot\widetilde Z_1)$,
let $Z_t=(1-t)Z_0+tZ_1$ and $\widetilde Z_t=(1-t)\widetilde Z_0+t\widetilde Z_1$.
With $U:=Z_1-Z_0=G\cdot\Delta$ and $\Delta:=\widetilde Z_1-\widetilde Z_0$, we have
\begin{equation}
\mathrm{Var}(U\mid Z_t)
=
\underbrace{\mathbb{E}\!\left[\mathrm{Var}(\Delta\mid \widetilde Z_t)\mid Z_t\right]}_{\text{within-slice difficulty}}
+
\underbrace{\mathrm{Var}\!\left(\mathbb{E}[U\mid Z_t,G]\mid Z_t\right)}_{\text{symmetry ambiguity}\ge 0}.
\end{equation}
Consequently,
\begin{equation}
\mathbb{E}\big[\mathrm{Var}(U\mid Z_t)\big]\;\ge\;\mathbb{E}\big[\mathrm{Var}(\Delta\mid \widetilde Z_t)\big].
\end{equation}
\end{theorem}

\begin{proof}
Apply the standard law of total variance with respect to the latent symmetry variable $G$:
\begin{equation}
\label{eq:ltv}
\begin{aligned}
\mathrm{Var}(U\mid Z_t)=
\mathbb{E}\!\left[\mathrm{Var}(U\mid Z_t,G)\mid Z_t\right]
+\mathrm{Var}\!\left(\mathbb{E}[U\mid Z_t,G]\mid Z_t\right).
\end{aligned}
\end{equation}
The second term is nonnegative. For the first term, conditioning on $(Z_t,G)$ determines $\widetilde Z_t = G^{-1}\cdot Z_t$. Under the lifted coupling, $(\Delta,\widetilde Z_t)$ is independent of $G$, and orthogonality implies $\mathrm{Var}(G\cdot X\mid \cdot)=\mathrm{Var}(X\mid \cdot)$. Therefore,
\begin{equation}
\mathrm{Var}(U\mid Z_t,G)=\mathrm{Var}(G\cdot\Delta\mid \widetilde Z_t,G)
=\mathrm{Var}(\Delta\mid \widetilde Z_t).
\end{equation}
Substituting into~\eqref{eq:ltv} yields~\eqref{eq:canon-var-decomp}. Taking expectations over $Z_t$ and using the tower property gives~\eqref{eq:canon-var-ineq}.
\end{proof}

The lower bound in Theorem~\ref{thm:canon-var-decomp} is the ``within-slice'' conditional variance; the gap is exactly the additional uncertainty induced by not knowing which symmetry element generated the observation. This formalizes the intuition that symmetry induces posterior multi-modality over group elements, which inflates the conditional variance and creates an irreducible error floor for coarse solvers.

\begin{remark}[What this does \emph{and does not} compare]
The lemma above compares ``with vs.\ without marginalizing $G$'' \emph{under the same lifted coupling}.
It does \emph{not} imply that a canonicalized training paradigm universally dominates any other paradigm under arbitrary choices of couplings. Its value is to isolate a nonnegative variance component that arises purely from symmetry ambiguity and is absent on the slice. However, as discussed in Proposition~\ref{prop:lift-product}, the lifted coupling can easily recover the standard product with Gaussian noises used in most practical flow matching models (w/o OT).
\end{remark}

In the next part we will further explain Remark~\ref{remark:non-equi-benefit}.
Fix a coupling $\gamma$ on $(Z_0,Z_1)$ and define the linear interpolation
\begin{equation}
Z_t=(1-t)Z_0+tZ_1,\qquad U:=Z_1-Z_0.
\end{equation}
For any predictor $v:\mathbb R^d\to\mathbb R^d$, the population squared loss at time $t$ is
\begin{equation}
\mathcal L_t(v):=\mathbb E\big[\|v(Z_t)-U\|^2\big].
\end{equation}
It is standard that the Bayes regressor is $v^\star(z)=\mathbb E[U\mid Z_t=z]$ and the Bayes risk equals the
conditional variance:
\begin{equation}
\inf_v \mathcal L_t(v)=\mathbb E\big[\mathrm{Var}(U\mid Z_t)\big].
\label{eq:bayes-risk-ambient}
\end{equation}

\paragraph{Canonicalization does not improve the ambient Bayes risk of equivariant models.}
\label{para:eq-no-gain-ambient}

Let a compact group $\mathcal G$ act orthogonally on $\mathbb R^d$ (e.g.\ permutations and rotations).
Let $\gamma$ be \emph{diagonal-invariant}:
\begin{equation}
(g,g)_{\#}\gamma=\gamma\quad \forall g\in\mathcal G,
\end{equation}
which holds in particular for the ambient product coupling $p_0\otimes\mathcal N(0,I)$ when $p_0$ is $\mathcal G$-invariant,
and also for symmetrized OT couplings.

\begin{proposition}[Equivariance of the Bayes regressor under diagonal invariance]
\label{prop:bayes-equivariant}
Assume $\mathcal G$ acts orthogonally and $\gamma$ is diagonal-invariant.
Then the Bayes regressor $v^\star(z)=\mathbb E[U\mid Z_t=z]$ is $\mathcal G$-equivariant:
\begin{equation}
v^\star(g\cdot z)=g\cdot v^\star(z)\qquad \forall g\in\mathcal G.
\end{equation}
Consequently, restricting to the equivariant function class
$\mathcal H_{\mathrm{eq}}:=\{v:\ v(g\cdot z)=g\cdot v(z)\}$ does \emph{not} increase the minimal population loss:
\begin{equation}
\inf_{v\in\mathcal H_{\mathrm{eq}}}\mathcal L_t(v)=\inf_v \mathcal L_t(v)=\mathbb E[\mathrm{Var}(U\mid Z_t)].
\end{equation}
\end{proposition}

\begin{proof}
Diagonal invariance implies $(Z_t,U)\stackrel{d}{=}(g\cdot Z_t, g\cdot U)$.
Hence for any measurable set $A$,
\begin{equation}
\mathbb E[U\mathbf 1_{\{Z_t\in A\}}]
=\mathbb E[g\cdot U\ \mathbf 1_{\{g\cdot Z_t\in A\}}]
=g\cdot \mathbb E[U\mathbf 1_{\{Z_t\in g^{-1}A\}}],
\end{equation}
which yields $v^\star(g\cdot z)=g\cdot v^\star(z)$ by Radon--Nikodym characterization of conditional expectation.
The equality of infima follows because $v^\star\in\mathcal H_{\mathrm{eq}}$.
\end{proof}

\begin{remark}
If one keeps the \emph{same ambient regression problem}~\eqref{eq:bayes-risk-ambient}, then canonicalizing the inputs does not improve the population optimum for an equivariant model class, because the Bayes regressor is already equivariant (Proposition~\ref{prop:bayes-equivariant}). 
\end{remark}

\paragraph{Canonical conditions expand the effective hypothesis class of non-equivariant models.}
\label{para:noneq-gain}

Canonicalization is most useful when the model is \emph{not} architecturally equivariant.
Intuitively, a non-equivariant model in ambient space must learn symmetry averaging from data.
Canonicalization and canonical conditions provide a \emph{gauge} (e.g.\ a canonicalizer $\Psi$ and an associated group element $\kappa(\cdot)$) that allows a non-equivariant model to implement equivariant behavior via an explicit formula.

We assume an exact canonicalizer model, i.e., a canonicalizer provides $\Psi:\mathbb R^d\to S$ (slice) and $\kappa:\mathbb R^d\to\mathcal G$ such that
$\Psi(x)=\kappa(x)^{-1}\cdot x$ on a full-measure set (up to stabilizers).

\begin{proposition}[Canonical-condition lifting realizes equivariant functions]
\label{prop:canon-lift-class}
Let $\mathcal A$ be any function class on the slice $S$ (not necessarily equivariant). Define the induced ambient class using the canonical condition:
\begin{equation}
\mathcal H_{\mathrm{lift}}(\mathcal A):=
\big\{\, h(x)=\kappa(x)\cdot a(\Psi(x))\ :\ a\in\mathcal A \,\big\}.
\end{equation}
Then every $h\in \mathcal H_{\mathrm{lift}}(\mathcal A)$ is $\mathcal G$-equivariant (on the set where $\Psi,\kappa$ are consistent).
Moreover, if $\mathcal A$ is universal on $S$, then $\mathcal H_{\mathrm{lift}}(\mathcal A)$ can approximate any equivariant target
function on $\mathbb R^d$ up to the usual stabilizer caveats.
\end{proposition}

\begin{proof}
Equivariance follows from $\Psi(g\cdot x)=\Psi(x)$ and $\kappa(g\cdot x)=g\kappa(x)$ (in the free-action regime):
\begin{equation}
h(g\cdot x)=\kappa(g\cdot x)\cdot a(\Psi(g\cdot x))=g\kappa(x)\cdot a(\Psi(x))=g\cdot h(x).
\end{equation}
Universality is inherited because any equivariant function is determined by its restriction to the slice.
\end{proof}

\begin{remark}[Why this is a real ``shortcut'' for non-equivariant networks]
A generic non-equivariant network that only sees $x$ must \emph{implicitly} learn the equivariant structure.
Providing $(\Psi(x),\kappa(x))$ (or a discrete proxy such as canonical rank/frame ID) allows implementing equivariance
by construction as in Proposition~\ref{prop:canon-lift-class}. This can reduce approximation complexity and improve optimization.
\end{remark}

To summarize, slice training does \emph{not} claim to reduce the ambient Bayes risk~\eqref{eq:bayes-risk-ambient}. The empirical Bayes risk that can be actually achieved still depends on the architectures and could not be improved for equivariant models (Proposition~\ref{prop:bayes-equivariant}).
Instead, it learns a \emph{different} regression problem on the slice (predicting $\Delta$ from $\widetilde Z_t$), whose irreducible variance is the first term in \Cref{thm:canon-var-decomp} and is provably no larger than the ambient one. This is where canonicalization can enable easier training and fewer steps sampling for non-equivariant models.
s

\subsubsection{Proof for \Cref{subsubsec:align_ot}}
\label{subsubsec:app_supp_323_compact}
This part provides supplementary to aligned slice priors, OT couplings, and train--test consistency.

Canonicalization eliminates the symmetry-ambiguity term, but the remaining within-slice conditional variance
$\mathbb{E}[\mathrm{Var}(\Delta\mid \widetilde Z_t)]$ depends on the \emph{slice prior} $q_1$ and the \emph{slice coupling}
$\widetilde\gamma$. 
\Cref{subsubsec:align_ot} explains that after canonicalization, the remaining irreducible term is the \emph{within-slice difficulty}
$\mathbb{E}[\operatorname{Var}(\Delta\mid \widetilde Z_t)]$, which depends on (i) the slice prior $q_1$ and (ii) the slice coupling
$\widetilde\gamma$ (Theorem~3.5). This appendix subsubsection collects three short technical supplements:
(i) a closed-form expression showing how \emph{misaligned} Gaussian priors can inflate within-slice variance under product coupling,
(ii) why OT/near-Monge couplings are complementary but can be non-unique under symmetry, and
(iii) a minimal consistency statement clarifying when training-time OT is compatible with inference and when conditioning causes mismatch.

\begin{figure}[t]
  \centering
  \includegraphics[width=\linewidth]{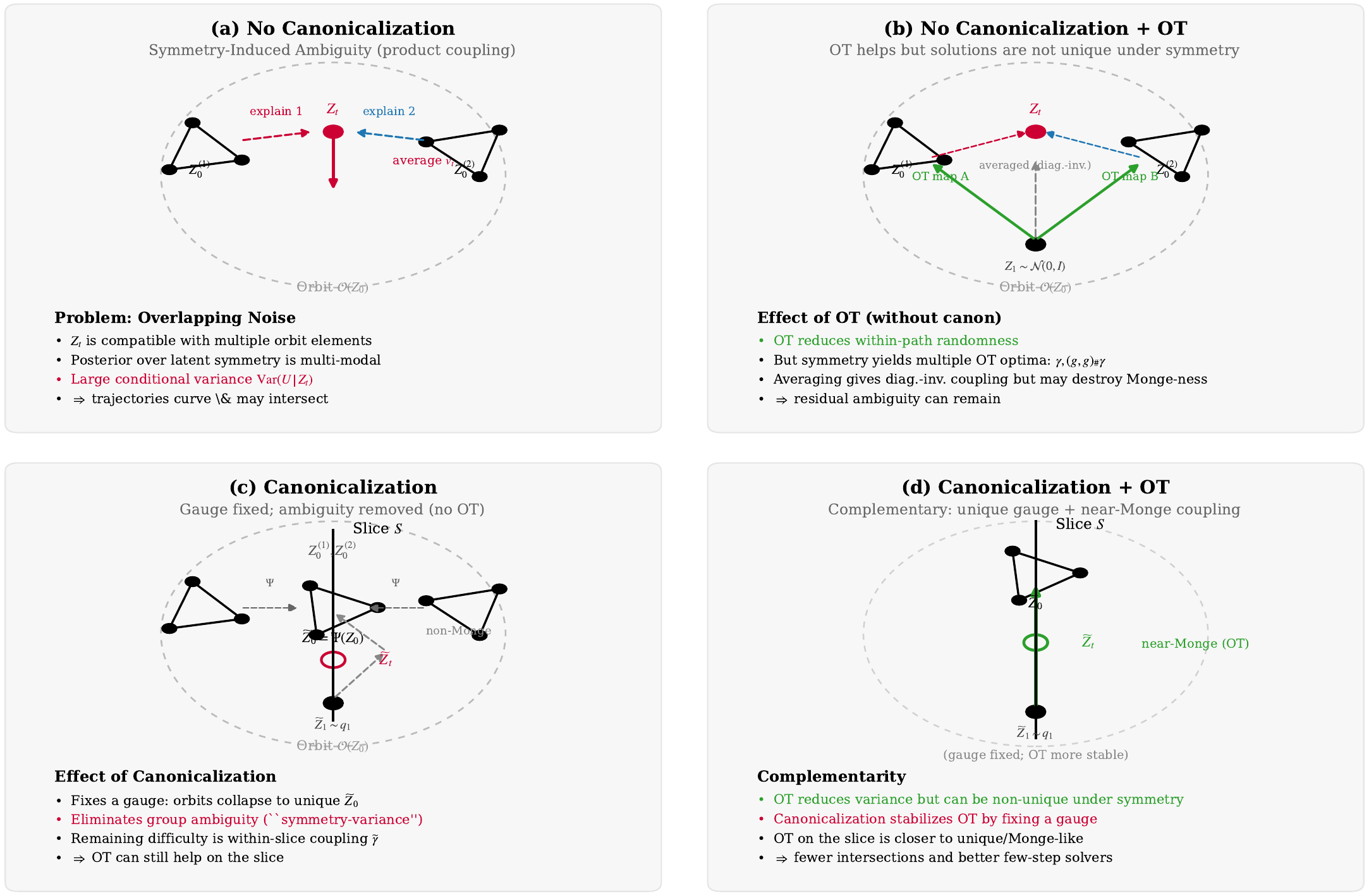}
  \caption{OT and canonicalization act complementarily. OT can reduce conditional variances with or without canonicalization, but OT solutions are generally non-unique in the presence of symmetry (if $\gamma$ is optimal then $(g,g)_\#\gamma$ is also optimal); averaging yields a diagonal-invariant optimum but may destroy Monge-ness. Canonicalization fixes a gauge (collapsing each orbit to a unique slice representative), eliminating group ambiguity and stabilizing OT on the slice.}
  \label{fig:ot-canon-complementary}
\end{figure}

\paragraph{Gaussian misalignment under product coupling: an exact within-slice variance formula.}
Canonicalization removes the symmetry-ambiguity term, but even on the slice the conditional variance can remain large if the
prior is poorly aligned with the slice geometry. The simplest (and common) baseline is the product coupling
$\widetilde\gamma=q_0\otimes q_1$. The following proposition gives an exact expression for the within-slice conditional covariance
in the linear-path (rectified-flow) setting, illustrating how mismatch between $\Sigma_0$ and $\Sigma_1$ controls the irreducible error.

\begin{proposition}[Closed-form within-slice conditional variance for independent Gaussians]
\label{prop:gaussian-product-condvar}
Assume the slice data and slice prior are Gaussian and independent:
\[
\widetilde Z_0\sim \mathcal{N}(\mu_0,\Sigma_0),\qquad
\widetilde Z_1\sim \mathcal{N}(\mu_1,\Sigma_1),\qquad
\widetilde Z_0 \perp \widetilde Z_1.
\]
Let $\widetilde Z_t = (1-t)\widetilde Z_0+t\widetilde Z_1$ and $\Delta=\widetilde Z_1-\widetilde Z_0$. Then
\begin{equation}
\mathrm{Cov}(\Delta\mid \widetilde Z_t)
= \frac{1}{t^2}\;\mathrm{Cov}(\widetilde Z_0\mid \widetilde Z_t),
\end{equation}
and the conditional covariance admits the closed form
\begin{equation}
\mathrm{Cov}(\widetilde Z_0\mid \widetilde Z_t)=
\Sigma_0 - (1-t)^2\;\Sigma_0\Big((1-t)^2\Sigma_0+t^2\Sigma_1\Big)^{-1}\Sigma_0.
\label{eq:gaussian-condcov}
\end{equation}
Consequently, the within-slice irreducible error in flow matching is
\begin{equation}
\mathbb{E}\big[\mathrm{Var}(\Delta\mid \widetilde Z_t)\big]=
\mathrm{tr}\!\left(
\frac{1}{t^2}\Sigma_0 - \frac{(1-t)^2}{t^2}\;\Sigma_0\Big((1-t)^2\Sigma_0+t^2\Sigma_1\Big)^{-1}\Sigma_0
\right).
\end{equation}
\end{proposition}

\begin{proof}
The identity $\Delta = (\widetilde Z_t - \widetilde Z_0)/t$ is algebraic, hence
$\mathrm{Cov}(\Delta\mid \widetilde Z_t)=\frac{1}{t^2}\mathrm{Cov}(\widetilde Z_0\mid \widetilde Z_t)$.
Since $(\widetilde Z_0,\widetilde Z_t)$ is jointly Gaussian with
$\mathrm{Cov}(\widetilde Z_0,\widetilde Z_t)=(1-t)\Sigma_0$ and
$\mathrm{Cov}(\widetilde Z_t)=(1-t)^2\Sigma_0+t^2\Sigma_1$, the standard Gaussian conditioning formula yields
\eqref{eq:gaussian-condcov}. Taking traces gives the final expression.
\end{proof}

Therefore, if one uses an isotropic prior $q_1=\mathcal N(0,I)$ while $q_0$ is strongly anisotropic (large condition number in $\Sigma_0$),
then the trace above remains large along directions where $\Sigma_0$ dominates $I$, formalizing why a \emph{misaligned} simple prior can
degrade few-step accuracy even after canonicalization. A practical remedy is to choose an \emph{aligned} prior within a tractable family,
e.g.\ the moment-matched Gaussian $q_1^\star=\mathcal N(\mathbb{E}_{q_0}[\widetilde Z_0],\,\operatorname{Cov}_{q_0}(\widetilde Z_0))$
(KL projection onto Gaussians), or a learned prior.
While it is possible to learn the canonical prior $q_1$, using KL projections of $q_0$ is a simple yet effective way. The computation only occurs in pre-processing, and does not induce any overhead in training.
For instance, among Gaussians, the moment-matched Gaussian is the KL projection of $q_0$:
\begin{proposition}[Moment-matched Gaussian is the KL-optimal Gaussian approximation]
\label{prop:kl-gaussian-projection}
Let $q_0$ be any distribution on $\mathbb{R}^d$ with finite second moments. Consider the family
$\mathcal{N}(\mu,\Sigma)$ with $\Sigma\succ 0$. The minimizer of $\mathrm{KL}(q_0\|\mathcal{N}(\mu,\Sigma))$ is
\begin{equation}
\mu^\star = \mathbb{E}_{q_0}[\widetilde Z_0],\qquad
\Sigma^\star = \mathrm{Cov}_{q_0}(\widetilde Z_0).
\end{equation}
\end{proposition}
\begin{proof}
This is a standard exponential-family projection: the Gaussian log-density is an affine function of the sufficient
statistics $(x,xx^\top)$, so minimizing $\mathrm{KL}(q_0\|\mathcal{N}(\mu,\Sigma))$ yields moment matching.
\end{proof}

\paragraph{OT/near-Monge couplings: complementary to canonicalization, but symmetry can destroy uniqueness.}
Given $(q_0,q_1)$, choosing $\widetilde\gamma$ close to an OT coupling can make transport more deterministic (Monge-like), shrinking
$\mathbb{E}[\operatorname{Var}(\Delta\mid \widetilde Z_t)]$ and straightening trajectories (hence benefiting few-step solvers).
However, when working in ambient space with symmetric marginals and symmetric costs, OT solutions are generally \emph{not unique}:
if $\gamma$ is optimal then $(g,g)_{\#}\gamma$ is also optimal. Averaging over $g$ yields a diagonal-invariant optimum, but can destroy
Monge-ness by turning a deterministic map into a mixture. Canonicalization fixes a gauge and thus \emph{stabilizes} OT: on the slice,
the symmetry-induced degeneracy is reduced, and OT solutions are empirically closer to unique/Monge-like. This is the precise sense in
which canonicalization and OT act \emph{complementarily}: canonicalization removes group ambiguity (Theorem~3.5), while OT targets the
remaining within-slice difficulty. \Cref{fig:ot-canon-complementary} provides an illustration.

\paragraph{Training-time OT is compatible with inference; mismatch comes from conditioning variables.}
A frequent confusion is whether using OT in training forces OT at inference. In flow matching / rectified flow, OT changes the \emph{training
coupling} $\widetilde\gamma$ (and thus the supervision signal), but inference only requires sampling from the \emph{marginal prior} $q_1$ and
integrating the learned dynamics. In particular, if the population vector field is well-defined and sufficiently regular, sampling is
\emph{coupling-free} given the correct field: one draws $\widetilde Z_1\sim q_1$ and integrates the learned ODE/SDE to obtain a sample from
$q_0$; no paired endpoints are needed at inference.

The actual train--test mismatch risk arises when the model is conditioned on an auxiliary variable $C$ (e.g.\ ``canonical rank'' PE, frame ID,
or any deterministic/stochastic output of a canonicalizer). Let $\widetilde\pi^{\mathrm{tr}}_1(C,\widetilde Z_1)$ be the training-time joint
law at the start, and $\widetilde\pi^{\mathrm{inf}}_1(C,\widetilde Z_1)$ the inference-time one. A sufficient condition for \emph{no mismatch} is:
\begin{equation}
\widetilde\pi^{\mathrm{tr}}_1(C,\widetilde Z_1)=\widetilde\pi^{\mathrm{inf}}_1(C,\widetilde Z_1).
\end{equation}
In particular, if $C$ is a deterministic index inherent to the canonical coordinate system (e.g.\ the coordinate index $i$ \emph{is} the canonical
rank), then this equality holds for any choice of $q_1$; conversely, if inference computes $C$ as a function of sampled noise in a way that differs
from training, mismatch is unavoidable. Finally, note that \emph{canonicalizing noise at inference} (e.g.\ setting $\widetilde Z_1:=\Psi(\varepsilon)$
with $\varepsilon\sim\mathcal N(0,I)$) generally induces a complicated slice prior $\Psi_{\#}\mathcal N(0,I)$; unless training uses the same induced
prior, this introduces an avoidable prior mismatch and can increase within-slice variance.

\section{Implementation Details}\label{sec_appendix:imple_details}

\subsection{Details of Canonicalization Methods}\label{subsec_appendix:canon}

\subsubsection{Geometric Spectral Ordering}

We define a deterministic canonicalization function $\kappa: \mathcal{M} \to \mathcal{M}$ that maps each molecular graph to a unique representative within its orbit under $S_N \times SE(3)$. The key insight is to leverage the geometric Laplacian constructed from 3D atomic coordinates, whose Fiedler vector encodes the molecular connectivity structure in a continuous, rotation-invariant manner.

\paragraph{Geometric Laplacian construction and canonical permutation.} The construction of the geometric Laplacian and the calculation of the (signed) Fiedler vector can be cleanly summarized as the pseudo-algorithm in \Cref{alg:geom-lap-fiedler}.

\begin{algorithm}[ht]
\caption{Geometric Laplacian and signed Fiedler vector}
\label{alg:geom-lap-fiedler}
\begin{algorithmic}
\STATE \textbf{Input:} coordinates $\mathbf{X} = (\mathbf{x}_1,\dots,\mathbf{x}_N) \in \mathbb{R}^{N\times 3}$, bond set $E$.
\STATE Compute mean bond length $\bar d_{\text{bond}} \gets \frac{1}{|E|}\sum_{(i,j)\in E} \|\mathbf{x}_i-\mathbf{x}_j\|$.
\STATE Set bandwidth $\sigma^2 \gets 4\,\bar d_{\text{bond}}^2$.
\FOR{$i=1,\dots,N$}
  \FOR{$j=1,\dots,N$}
    \IF{$i=j$}
      \STATE $W_{ij} \gets 0$.
    \ELSE
      \STATE $W_{ij} \gets \exp\big(-\|\mathbf{x}_i-\mathbf{x}_j\|^2/(2\sigma^2)\big)$.
    \ENDIF
  \ENDFOR
\ENDFOR
\STATE Degree matrix $\D \gets \mathrm{diag}(\sum_j W_{ij})$.
\STATE Random-walk Laplacian $L_{\mathrm{rw}} \gets \D^{-1}(\D-\W)$.
\STATE Compute eigenpairs of $L_{\mathrm{rw}}$: $0=\lambda_1\le \lambda_2\le \cdots$ with eigenvectors $\{u_k\}$.
\STATE Fiedler vector $u_2 \gets$ eigenvector of $\lambda_2$.
\STATE Centroid $\bar{\mathbf{x}} \gets \frac{1}{N}\sum_i \mathbf{x}_i$ and mean radius $\bar d \gets \frac{1}{N}\sum_i \|\mathbf{x}_i-\bar{\mathbf{x}}\|$.
\STATE Fix sign by centroid-direction convention:
\STATE \hspace{1em}$u_2 \gets \mathrm{sign}\big(\sum_i u_{2,i}(\|\mathbf{x}_i-\bar{\mathbf{x}}\| - \bar d)\big)\cdot u_2$.
\STATE \textbf{Output:} signed Fiedler vector $u_2$.
\end{algorithmic}
\end{algorithm}

\begin{algorithm}[htbp]
\caption{Spectral $SO(3)$ Canonicalization}
\label{alg:spectral-so3-canon}
\begin{algorithmic}[1]
\REQUIRE Atomic coordinates $\mathbf{X} \in \mathbb{R}^{N \times 3}$ and signed Fiedler vector $v$ (computed by Algorithm~\ref{alg:geom-lap-fiedler})
\ENSURE Canonicalized coordinates $\mathbf{X}'$
\STATE Apply spectral ordering $\pi^* = \operatorname{argsort}(v)$ and enforce sign $\sum_i v_i^3 > 0$
\STATE Let head index $h$ be rank $0$ and tail index $t$ be rank $N-1$
\STATE Choose anchor index
$$
a = \operatorname{argmax}_{k \in [\lfloor N/3 \rfloor, \lfloor 2N/3 \rfloor)} \|(\mathbf{x}_k - \mathbf{x}_{h}) \times (\mathbf{x}_{t} - \mathbf{x}_{h})\|
$$
\STATE Compute longitudinal axis $\mathbf{e}_1 \gets \frac{\mathbf{x}_{t} - \mathbf{x}_{h}}{\|\mathbf{x}_{t} - \mathbf{x}_{h}\|}$
\STATE Compute plane normal $\mathbf{n} \gets \mathbf{e}_1 \times (\mathbf{x}_{a} - \mathbf{x}_{h})$
\STATE Normalize normal $\mathbf{e}_3 \gets \frac{\mathbf{n}}{\|\mathbf{n}\|}$
\STATE Complete right-handed basis $\mathbf{e}_2 \gets \mathbf{e}_3 \times \mathbf{e}_1$
\STATE Assemble rotation matrix $\mathbf{R} \gets [\mathbf{e}_1, \mathbf{e}_2, \mathbf{e}_3]^\top$
\STATE Center and rotate $\mathbf{X}' \gets (\mathbf{X} - \bar{\mathbf{X}})\mathbf{R}^\top$
\end{algorithmic}
\end{algorithm}

\begin{figure*}[t]
  \centering
  \begin{subfigure}[t]{0.6\textwidth}
    \centering
    \includegraphics[width=\linewidth]{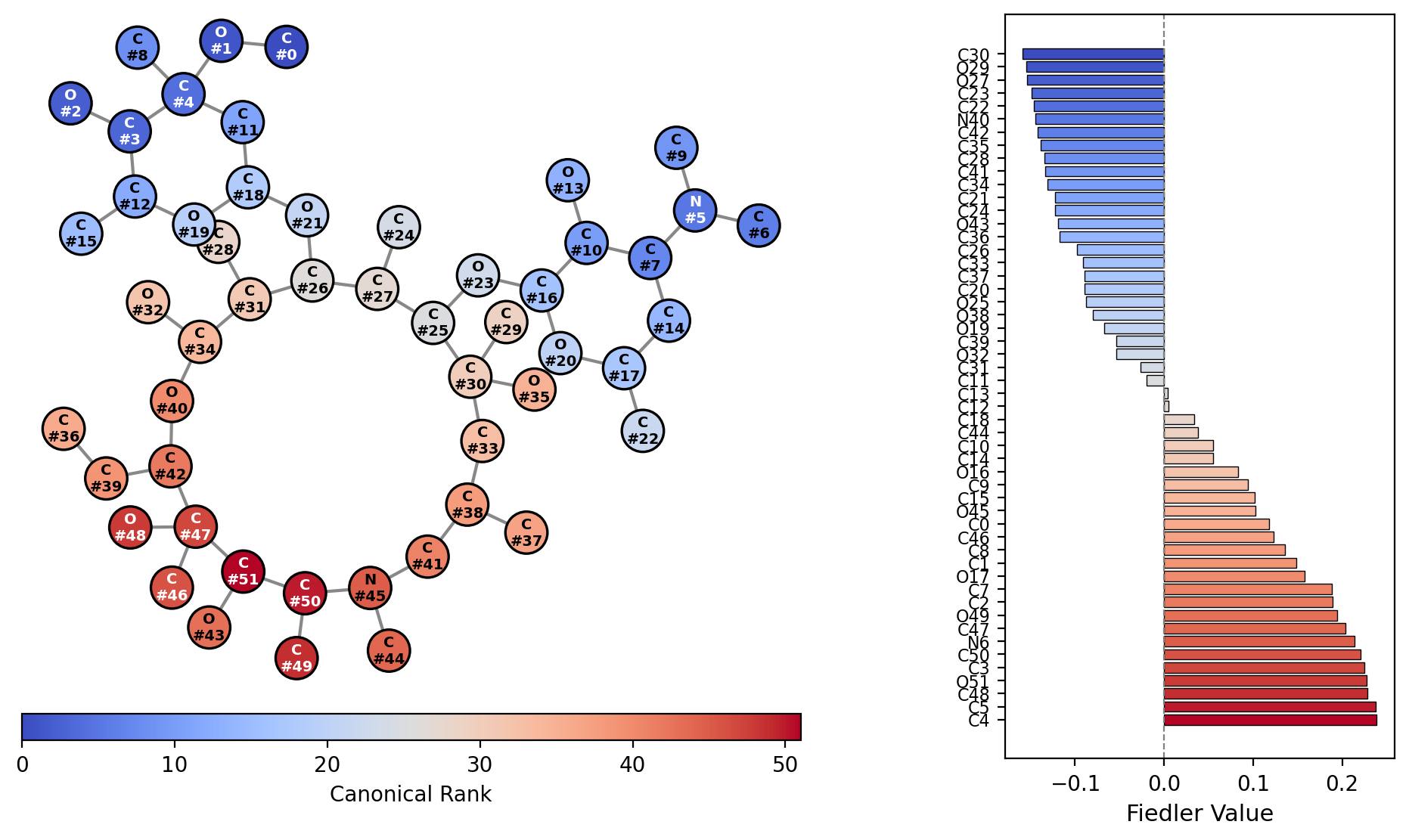}
    \caption{Azithromycin (high symmetry / macrocycle): the spectral ordering is core-centric, concentrating the macrocycle near the spectral center and mapping peripheral sugar moieties toward the sequence terminals.}
    \label{fig:spectral-canon-azithro}
  \end{subfigure}
  \hfill
  \begin{subfigure}[t]{0.6\textwidth}
    \centering
    \includegraphics[width=\linewidth]{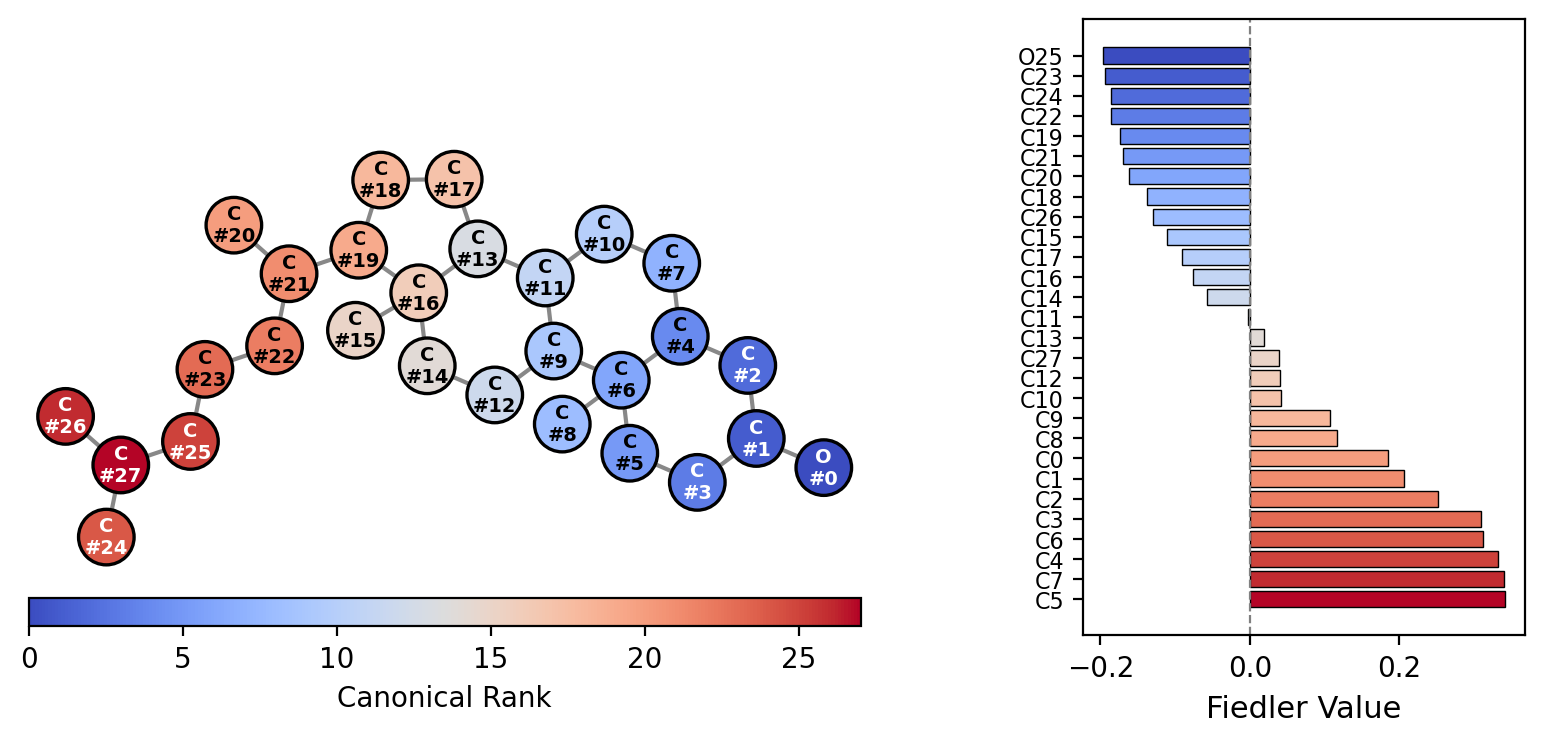}
    \caption{Cholesterol (low symmetry / anisotropic): the spectral ordering unfolds along the principal axis from the polar head through the rigid tetracyclic core to the flexible hydrocarbon tail, while preserving local connectivity.}
    \label{fig:spectral-canon-chol}
  \end{subfigure}
  \caption{Spectral canonicalization via the (signed) Fiedler vector produces stable, locality-preserving atom orderings for both highly symmetric and weakly symmetric molecules (blue: start, red: end), demonstrating robust performance across diverse symmetry regimes.}
  \label{fig:spectral-canonicalization}
  \vspace{-2pt}
\end{figure*}

The canonical ordering is then: $\pi^* = \text{argsort}(u_2)$
and the canonicalized molecule is $\kappa(\mathcal{M}) = \pi^*(\mathcal{M})$.
The Fiedler vector approximately captures the graph Laplacian's fundamental mode of oscillation, providing a ``core-to-periphery" ordering. Atoms with similar Fiedler values tend to be spatially close and structurally equivalent, which induces a natural generation order from molecular core to functional groups; see \Cref{fig:spectral-canonicalization} for illustrative examples.

\paragraph{Canonical $SO(3)$ frame.} The above procedure determines a canonical order on $S_N$, based on which we further (optionally) process the $SO(3)$ symmetry using \Cref{alg:spectral-so3-canon}.

\subsubsection{Alternative Orderings} 
Graph canonicalization is a widely studied topic and there are also other existing canonicalization methods. However, most of them are defined for permutation in abstract graphs~\citep{zhao2024pard,LaplacianCanonization,RethinkingGraphCanon} instead of geometric graphs.

We also implement two simpler orderings for ablation:
\begin{enumerate}
    \item Structural Multihop (inspired by PARD~\citep{zhao2024pard}): Iterative degree peeling using weighted multihop degrees $w_{K}(v) = \sum_{k=1}^{K} d_k(v) \cdot N^{K-k}$, where $d_k(v)$ counts nodes at exactly $k$ hops.
    \item Atomic Numbering: Priority-based ordering by atomic number (e.g., heavy/rare atoms first, hydrogen last).
\end{enumerate}

However, we experimentally find that our structure-aware geometric spectral ordering outperforms all these methods, validating the effectiveness.

\subsection{Canonical Diffusion and CanonFlow}\label{subsec_appendix:arch} 

We present additional algorithmic and architectural design for the canonical diffusion and flow matching.

\subsubsection{Details of canonical diffusion and flow matching}
We already describe the overall framework of canonical diffusion or flow matching in \Cref{subsec:algorithm}. We now provide additional details in this part.

\paragraph{Canonical positional encoding.} To explicitly break permutation equivariance, we inject positional information derived from the canonical rank into the model architecture. For each atom $i$ with normalized rank $r_i = \text{rank}_i / N \in [0, 1)$, we compute a sinusoidal positional encoding:
\begin{equation}
\text{PE}(r, 2k) = \sin\left(\frac{r \cdot M}{10000^{2k/d_{\text{pe}}}}\right), \qquad \text{PE}(r, 2k+1) = \cos\left(\frac{r \cdot M}{10000^{2k/d_{\text{pe}}}}\right)
\end{equation}
where $M = 10000$ is the max position scale and $d_{\text{pe}}$ is the encoding dimension. This encoding is concatenated with atom features before the initial projection.
This completely breaks permutation equivariance at the architecture level, allowing the model to distinguish atoms by their structural role. Note that this PE can be applied to either our novel Canon architecture, or other existing architectures to make a (non-equivariant) canonicality-aware counterpart.
The \emph{PE-drop} mechanism replaces $\mathbf R$ with a learned ``fake PE'' embedding with some probability $p_{\text{drop}}$ during training, which is used for classifier-free guidance w.r.t.\ canonical conditions. 

\paragraph{Adaptive position-dependent prior.} Standard flow matching samples noise from a uniform prior $p_0$. We replace this with a position-dependent prior that encodes empirical correlations between canonical position and atom type:
\begin{equation}
    p_0^{\text{eff}}(c \mid r) = \beta_r \cdot p_{\text{prior}}(c) + (1 - \beta_r) \cdot p_r(c \mid r)
\end{equation}
where $r = r_i / N$ is the relative position, $\beta_r$ is a mixing coefficient (typically $0.1$ for categorical ), $p_{\text{prior}}$ is the prior distribution such as $p_{\text{uniform}}$, and $p_r(\cdot|r)$ is the empirical distribution estimated from the training set.

For \emph{prior estimation}, we discretize $r \in [0, 1)$ into $K$ bins and compute:
\begin{equation}P(c \mid B_k) = \frac{\text{Count}(c, k) + \epsilon}{\sum_{c'} (\text{Count}(c', k) + \epsilon)}\end{equation}
for categorical data; we use a Gaussian with sufficient statistics mean and variance in the dataset for continuous coordinates.
At runtime, we use linear interpolation between adjacent bins for continuous $r$:
\begin{equation}
    p_r(c|r) = (1 - \delta) \cdot P(c | B_{k_0}) + \delta \cdot P(c | B_{k_1})
\end{equation}
where $k_0 = \lfloor r \cdot K \rfloor$, $k_1 = \min(k_0 + 1, K-1)$, and $\delta = r \cdot K - k_0$.

\subsubsection{Canon Architecture}

The Semla architecture proposed by \citep{irwin2024semlaflow} maintains two hidden states corresponding to $\Xc, \Hc$ in each layer. In our \emph{Canon} architecture, we additionally incorporate a third hidden state to update and refine canonical information in each layer by interacting with other atom and bond features.
\Cref{fig:canon-arch} summarizes the Canon architecture. 

\paragraph{Input processing.}
We implement a three-stream molecular dynamics network, \textbf{Canon}, which augments a Semla-like equivariant message passing trunk with an explicit \emph{canonical-rank stream} and optional canonical positional encodings (PE). Concretely, at diffusion/flow time $t$, the model takes as input
\[
\Xc_t\in\mathbb{R}^{B\times N\times 3},\qquad
\Hc_t\in\mathbb{R}^{B\times N\times d_{h}},\qquad
\mathbf R\in\mathbb{R}^{B\times N},
\]
where $\Xc_t$ are coordinates, $\Hc_t$ are node-wise invariant features (e.g.\ atom features, optional conditioning), and $\mathbf R$ is the canonical rank (or its normalized variant). The positional encoding processing is identical to the general description above. Optional self-conditioning provides a previous estimate $\hat \Xc_0$, and optionally a previous rank estimate $\hat{\mathbf R}$. In the implementation, self-conditioning is realized by stacking $(\Xc_t,\hat{\Xc}_0)$ and projecting into $K$ coordinate sets via a linear map $\Pi_{CS}$:
\begin{equation}
\mathbf{CS}_t \;=\;\Pi_{CS}\big([\Xc_t,\hat{\Xc}_0]\big)\in\mathbb{R}^{B\times K\times N\times 3},
\end{equation}
with masking applied on padded atoms. 

We append a size embedding $\mathrm{Emb}(|V|)$ and optional canonical PE to the invariant features and project them into the trunk width $d$:
\begin{equation}
\widetilde \Hc_t \;=\; \phi_{\text{in}}\Big([\Hc_t,t,\mathrm{Emb}(|V|),\mathrm{PE}(\mathbf R)]\Big)\in\mathbb{R}^{B\times N\times d},
\end{equation}
where $\phi_{\text{in}}$ is a two-layer MLP with SiLU nonlinearity. 

\paragraph{CanonDynamics: message-passing layers with node/coord/rank streams.}
The trunk (\texttt{CanonDynamics}) stacks $L$ copies of a message passing layer, optionally with edge features. Each layer updates three coupled states:
\[
(\mathbf{CS},\Hc,\mathbf R)\mapsto(\mathbf{CS}',\Hc',\mathbf R'),
\]
In each \emph{edge message layer}, we construct pairwise messages using \emph{(i)} projected node features, \emph{(ii)} projected rank features, and \emph{(iii)} geometric features derived from coord-set dot-products. Specifically, after normalization we compute
\begin{equation}
\mathbf G_{ij}^{(k)} \;=\;\langle \mathbf{CS}^{(k)}_i,\; \mathbf {CS}^{(k)}_j\rangle,\qquad k=1,\dots,K,
\end{equation}
and concatenate
\begin{equation}
\bm m_{ij} \;=\; \mathrm{MLP}\Big(\,[\W_h \bm h_i,\; \W_h \bm h_j,\; \W_r \bm r_i,\; \W_r \bm r_j,\; \mathbf{G}_{ij}^{(1:K)},\; \bm{e}_{ij}]\,\Big),
\end{equation}
where $\bm{e}_{ij}$ is an optional edge feature input, and $\W_h, \W_r$ are learnable projection weights.

Next we adopt attention-based updates for node/coord/rank features.
The message tensor is split into three channels (plus optional edge-out):
\begin{equation}
\bm{m}_{ij} \;=\; \big(\bm{m}^{\mathrm{node}}_{ij},\; \bm{m}^{\mathrm{coord}}_{ij},\; \bm m^{\mathrm{rank}}_{ij},\; \bm m^{\mathrm{edge}}_{ij}\big),
\end{equation}
which drive three attention modules:
\begin{equation}
\Hc \leftarrow \Hc + \mathrm{Attn}_{\mathrm{node}}(\Hc, \bm m^{\mathrm{node}}),\quad
\mathbf {CS} \leftarrow \mathbf{CS} + \mathrm{Attn}_{\mathrm{coord}}(\mathbf{CS}, \bm m^{\mathrm{coord}}),\quad
\mathbf R \leftarrow \mathbf R + \mathrm{Attn}_{\mathrm{rank}}(\mathbf R, \bm m^{\mathrm{rank}}),
\end{equation}
interleaved with feed-forward residual blocks consisting of equivariant or invariant MLPs for $(\Hc,\mathbf {CS})$ and $\mathbf R$. 

\paragraph{Output heads.}
After $L$ layers, Canon produces coordinate updates through a normalization + linear head:
\begin{equation}
\widehat \Xc \;=\; \mathrm{Head}_X\big(\mathrm{Norm}(\mathbf{CS})\big)\in\mathbb{R}^{B\times N\times 3},
\end{equation}
and predicts atom type and charge logits via linear classifiers on the final node features. If enabled, it also predicts a rank score per node via a linear head and min--max normalization to $[0,1]$. Optional edge features are refined by a bond-refinement module and projected to edge-type logits.

\begin{figure}[t]
  \centering
  \includegraphics[width=\linewidth]{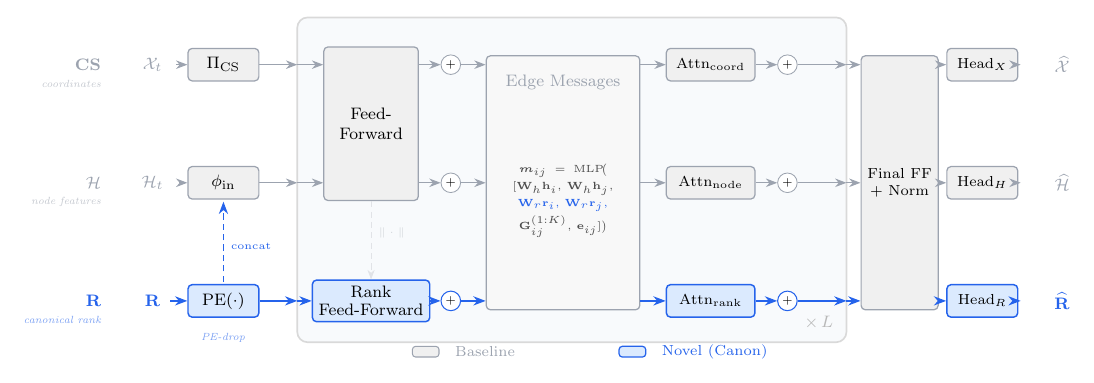}
  \caption{Overview of the Canon architecture. Three parallel streams---coordinates ($\mathbf{CS}$), node features ($\mathcal{H}$), and canonical rank ($\mathbf{R}$)---are updated through $L$ message-passing layers consisting of feed-forward blocks, pairwise edge messages, and attention-based aggregation. Gray blocks denote components inherited from Semla; blue blocks denote novel canonical-rank components introduced by Canon.}
  \label{fig:canon-arch}
\end{figure}

\subsection{Training}\label{subsec_appendix:training}

\paragraph{Rank noise for robustness.} To reduce the train-inference gap (since ground-truth ranks are unavailable during sampling), we add noise to canonical ranks during training:
\begin{equation}
r_{\text{noised}} = r + \mathcal{N}(0, \sigma_r^2) \cdot (1 - t)
\end{equation}
The noise decays with $t$ to avoid large perturbations when the molecule is nearly reconstructed.

\paragraph{Auxiliary rank prediction.} Optionally, the model predicts canonical rank as an auxiliary task:
\begin{equation}
\hat{r} = \sigma\big(\text{MLP}(h_i^{(L)})\big) \in [0, 1]
\end{equation}
with min-max normalization across atoms and training loss:
\begin{equation}
\mathcal{L}_{\text{rank}} = \lambda_r \cdot \mathbb{E}\left[\|r^* - \hat{r}\|^2\right]
\end{equation}
This enables dynamic rank estimation during inference.

\begin{algorithm}[ht]
\caption{Gap-free canonical training with optional OT coupling (no OT needed at inference)}
\label{alg:gapfree}
\begin{algorithmic}[1]
\REQUIRE Dataset $(G,X)$; symmetry group $G\in\{S_N,\;S_N\times SO(3)\}$; canonicalizer $\Psi_G$ producing rank-ordered canonical coordinates $\widetilde X$.
\STATE \textbf{Group setup:} choose $\Psi_G$ and the canonical-space prior $q_1$ accordingly.
\STATE Canonicalize each training example once: $\widetilde Z_0=\mathrm{vec}(\Psi(X))$; treat rank as the fixed coordinate index.
\STATE Choose a marginal prior $q_1$ in canonical space (Aligned or Isotropic).
\FOR{each training iteration}
  \STATE Sample minibatch $\{\widetilde z_0^i\}\sim q_0$ and noise minibatch $\{\widetilde z_1^j\}\sim q_1$.
  \STATE Construct a coupling:
    \begin{itemize}
    \item Product: pair $(\widetilde z_0^i,\widetilde z_1^i)$; or
    \item OT/Sinkhorn: compute plan $\Pi$ minimizing $\sum_{i,j}\Pi_{ij}\|\widetilde z_0^i-\widetilde z_1^j\|^2$, then sample pairs from $\Pi$.
    \end{itemize}
  \STATE Train the flow/diffusion model on canonical coordinates using fixed/predicted rank embeddings $\mathrm{PE}(i)$.
\ENDFOR
\end{algorithmic}
\end{algorithm}

\paragraph{Optimal transport annealing.} We observe that optimal transport (OT) matching between noise and target molecules reduces flow variance but may limit generalization. We implement \textbf{OT annealing}: the probability of using OT decreases linearly during training:
\begin{equation}
p_{\text{OT}}(\text{epoch}) = \max\left(0, 1 - \frac{\text{epoch}}{\text{max epochs}}\right)
\end{equation}
This allows the model to benefit from OT alignment early in training while learning to handle arbitrary noise-target pairings later.

\paragraph{Training objective.}
The overall training loss combines coordinate regression, categorical cross-entropy for discrete features, and optional rank prediction:
\begin{align}
    \mathcal{L}_{\text{coord}} &= \frac{1}{N} \sum_{i=1}^{N} \|\hat{\X}_{1,i} - \X_{1,i}\|^2, \qquad 
    \mathcal{L}_{\text{type}} = \frac{1}{N} \sum_{i=1}^{N} \text{CE}(\hat{\mathbf{h}}_{1,i}, \mathbf{h}_{1,i}) \\
    \mathcal{L}_{\text{bond}} &= \frac{1}{N^2} \sum_{i,j} \text{CE}(\hat{\mathbf{E}}_{1,ij}, \mathbf{E}_{1,ij}), \qquad 
    \mathcal{L}_{\text{rank}} = \frac{1}{N} \sum_{i=1}^{N} (\hat{r}_i - r_i^*)^2\\
    \mathcal{L} &= \mathcal{L}_{\text{coord}} + \lambda_{\text{type}} \mathcal{L}_{\text{type}} + \lambda_{\text{bond}} \mathcal{L}_{\text{bond}} + \lambda_{\text{charge}} \mathcal{L}_{\text{charge}} + \lambda_{\text{rank}} \mathcal{L}_{\text{rank}}
\end{align}
where default weights are $\lambda_{\text{type}} = 0.2 $, $ \lambda_{\text{bond}} = \lambda_{\text{charge}} = 1.0$ and $\lambda_{\text{rank}} = 0.1$. Following \cite{irwin2024semlaflow}, we set $\lambda_{\text{bond}}=0.5$ for QM9 dataset.

The complete training procedure can be summarized into \Cref{alg:gapfree}.

\subsection{Sampling}

\begin{algorithm}[ht]
\caption{Unified few-step canonical sampling (Regime A/B + options)}
\label{alg:unified-sampling}
\begin{algorithmic}[1]
\REQUIRE Symmetry group $G\in\{S_N,\;S_N\times SO(3)\}$; trained model $v_\theta$; baseline prior $q_1$; step times $1=t_K>\cdots>t_0=0$.
\REQUIRE Sampling regime flag $\mathsf{Regime}\in\{\mathrm{A},\mathrm{B}\}$.
\REQUIRE Options: (i) aligned prior; (ii) PCS (if enabled, apply $\Psi$ after each step; requires $\mathsf{Regime}=\mathrm{B}$).
\STATE \textbf{Note (rank conditioning):} if rank is used as a condition/PE, treat it as the \emph{fixed coordinate index} and \emph{do not} recompute it from noise.
\STATE \textbf{(Optional prior choice):} optionally replace $q_1$ with an alternative prior.
\STATE Sample $\widetilde Z_{t_K}\sim q_1$ in the fixed canonical coordinate system.
\FOR{$k=K,\dots,1$}
  \STATE $\widetilde Z_{t_{k-1}}\leftarrow \mathrm{Step}(\widetilde Z_{t_k}, v_\theta, t_k\to t_{k-1})$ \COMMENT{ODE/SDE solver}
  \IF{$\mathsf{Regime}=\mathrm{B}$}
    \STATE $\widetilde Z_{t_{k-1}}\leftarrow \Psi(\widetilde Z_{t_{k-1}})$
  \ELSE
    \STATE \textbf{Note:} do \emph{not} re-canonicalize / re-rank / project during sampling in Regime~A; keep rank indices fixed.
  \ENDIF
\ENDFOR
\STATE Output $\widetilde Z_{t_0}$.
\STATE \textbf{Optional invariance restoration:} sample $g\sim\Haar(G)$ and output $g\acts \widetilde Z_{t_0}$.
\end{algorithmic}
\end{algorithm}

\paragraph{Recovering invariance through Haar randomization.}
The critical observation is that canonical samples, while not themselves drawn from an invariant distribution, can be \emph{lifted} to an invariant distribution through a simple post-processing step. As an immediate result of \Cref{thm:factor}, the following theorem, which generalizes the observation in \cite{yan2023swingnn} to arbitrary compact groups, provides the theoretical guarantee in sampling invariant distributions:
\begin{proposition}[Post-hoc randomization yields invariant sampling]\label{prop:randomize}
Let $\tilde\mu$ be any distribution on $\M$. Define the randomized distribution
\begin{equation}
\mu := \int_{\G} (g\acts)\push \tilde\mu \,\mathrm{d}\Haar(g).
\end{equation}
Then $\mu$ is $\G$-invariant. Moreover, sampling $g\sim\Haar$ and $\tilde\Z\sim\tilde\mu$ and outputting $g\acts \tilde\Z$ produces $\mu$.
\end{proposition}

\begin{proof}
For any $h\in\G$,
\begin{equation}
(h\acts)\push \mu = \int_{\G} (h g\acts)\push \tilde\mu\,\mathrm{d}\Haar(g)
= \int_{\G} (g'\acts)\push \tilde\mu\,\mathrm{d}\Haar(g')=\mu,
\end{equation}
where we changed variables $g'=hg$ and used left-invariance of Haar. \qedhere
\end{proof}

\begin{remark}[Canonicalization vs. randomization]
Randomization alone enforces invariance but does not simplify the learning problem. Canonicalization enforces a \emph{gauge choice} that can substantially reduce multimodality/variance induced by symmetry and thereby accelerate training convergence, while randomization then restores invariance at the end. 
The combination of canonicalized training with post-hoc randomization thus achieves the best of both worlds: efficient learning and provably invariant generation. 
\end{remark}

\paragraph{Dynamic rank estimation.} Since sampling can be done using a fixed canonical rank as the condition, the dynamic updates of rank estimation is completely optional. 

A practical few-step stabilizer is to project intermediate states back to the slice (Regime B in \Cref{alg:unified-sampling}). In particular, to keep rank conditions and partially noisy samples consistent, we allow projection-to-slice sampling, which we termed as Projected Canonical Sampling (PCS) as detailed in \Cref{alg:pcs}. Note that projecting the data to the canonical slice while keeping the PE ranks fixed is conceptually equivalent to keeping the data in the original order but update the canonical rank to align the two.

We support two strategies to estimate canonical rank in the absence of ground truth in \Cref{alg:pcs}:
\begin{enumerate}
    \item Predict Mode: Use the model's rank prediction output $\hat{r}$ directly for warping.

    \item Canonicalize Mode: Periodically recompute canonical rank from intermediate predictions: Discretize current atom type predictions: $\hat{a}_i = \text{argmax}(\text{atomics}_i)$. Then apply the same canonicalization algorithm to the predicted molecule. Update rank every $K$ steps when $t \geq T$.
\end{enumerate}
We experimentally find that the predict mode is always better, likely due to the instability of canonicalization algorithms for intermediate noisy states; instead, self-prediction integrates this procedure into the model learning, facilitating robustness and expressivity.

Remarkably, if the dynamic rank estimation is disabled (Regime A in \Cref{alg:unified-sampling}), the model always conducts a ``conditional generation" task given the canonical rank conditions, which is also a valid task with no train-test gap.

The complete sampling procedure can be summarized into \Cref{alg:unified-sampling}.

\begin{algorithm}[ht]
\caption{Projected Canonical Sampling (PCS)}\label{alg:pcs}
\begin{algorithmic}[1]
\REQUIRE Canonicalizer $\Psi_\phi$; learned slice model (score or vector field) $m_\theta$; time grid $1=t_K>\cdots>t_0=0$.
\STATE Initialize $\tilde\Z_{t_K}\sim q_1$.
\FOR{$k=K,\dots,1$}
  \STATE Take one reverse step on the slice: $\widehat{\tilde\Z}\gets \mathrm{Step}(\tilde\Z_{t_k},m_\theta,t_k\to t_{k-1})$.
  \STATE Project: $\tilde\Z_{t_{k-1}}\gets \Psi_\phi(\widehat{\tilde\Z})$.
\ENDFOR
\STATE Output canonical sample $\hat\Z_0=\tilde\Z_{t_0}$ and optionally randomize by $g\sim\Haar$.
\end{algorithmic}
\end{algorithm}

\section{Experimental Details and Additional Results}\label{sec_appendix:exp}

\subsection{Experimental Details}

\subsubsection{Evaluation metrics}

We evaluate generated molecules using the following metrics. Let $\mathcal{M} = \{m_i\}_{i=1}^N$ denote the set of generated molecules, and let $\tilde{m}_i$ denote the RDKit force-field optimised conformation of $m_i$.
\begin{itemize}
    \item \textbf{Validity.} The fraction of generated molecules that pass RDKit chemical validity checks.
    \item \textbf{Atom Stability.} The proportion of all generated atoms whose explicit valence lies within a predefined allowed range.
    \item \textbf{Molecule Stability.} The fraction of generated molecules in which \emph{every} atom is stable.
    \item \textbf{Uniqueness.} Among molecules that can be converted to canonical SMILES, the fraction of distinct SMILES strings.
    \item \textbf{Novelty.} Among molecules with valid canonical SMILES, the fraction absent from the training set.
    \item \textbf{Opt-RMSD.} The average root-mean-square deviation between generated and optimised conformations:
    \[
\text{Opt-RMSD}
= \frac{1}{|\mathcal{I}|}
  \sum_{i \in \mathcal{I}} \mathrm{RMSD}(m_i,\, \tilde{m}_i),
\]
where $\mathcal{I} = \{i : \text{optimisation succeeds and RMSD is finite}\}$.
This quantifies how far the generated 3D geometry deviates from a local energy minimum.
    \item \textbf{NFE.} The Number of Function Evaluations, i.e.\ the total number of neural network forward passes required by the ODE solver during sampling. NFE serves as a hardware-agnostic measure of generation cost.
    \item \textbf{Sampling Time.} The wall-clock time to generate the full evaluation set ($N{=}1{,}000$ molecules).
\end{itemize}

\subsubsection{Hyper-parameters.}

Our most hyper-parameters completely follow \citet{irwin2024semlaflow} for both Canonical SemlaFlow and our CanonFlow. All of our backbones consist of $L{=}12$ equivariant message-passing layers
($d_\text{model}{=}384$, $d_\text{msg}{=}128$, $d_\text{edge}{=}128$,
$H{=}32$ attention heads, $S{=}64$ coordinate sets with per-set length normalisation), with two edge-aware layers at the input and output.
A learnable molecule-size embedding of dimension~64 is concatenated to each atom's input features.
All output heads (atom type, charge, bond type, and the optional canonical-rank head) are two-layer MLPs with SiLU activation. Self-conditioning is employed by default: during training, with probability~$0.5$ the model first produces a preliminary prediction with zeroed conditioning inputs, which is then fed back as additional context; during inference it is applied at every integration step.

Models are trained for 200 epochs on GEOM-Drugs and 300 epochs on QM9
with Adam ($\text{lr}{=}3{\times}10^{-4}$, \texttt{amsgrad}, no weight decay),
a linear warm-up (10\,000 steps for GEOM-Drugs, 2\,000 for QM9) followed by a constant learning rate,
gradient clipping at global norm~1.0,
and bucket-based dynamic batching with a cost budget of~4096 and linear cost scaling.
An exponential moving average (EMA, decay~$0.999$) of model parameters is maintained and used for all inference.
All training is conducted in FP32 precision. The total training time is approximately 10 hours for QM9 and 28 hours for GEOM-DRUG on a single NVIDIA A100 80GB gpu. 

Coordinates are interpolated linearly,
$\mathbf{x}_t = (1{-}t)\,\mathbf{x}_0 + t\,\mathbf{x}_1$,
with additive Gaussian noise ($\sigma{=}0.2$).
Atom types and bond types use the uniform-sample categorical interpolation~\cite{campbell2022discrete}.
The interpolation time is drawn from $\mathrm{Beta}(2,1)$.
Equivariant optimal transport (Kabsch alignment + Hungarian matching) is optionally applied to coordinate pairs.
When classifier-free guidance (CFG) is enabled, the positional-encoding dropout rate defaults to $p_{\text{drop}}{=}0.1$.


\begin{table}[t]
  \caption{CanonFlow w/ our novel Canon architecture on GEOM-DRUG. Canonicalization is conducted on $S_N$.}
  \label{table:cfg_rank3}
  \centering
  \begin{tabular}{lcccc}
    \toprule
    Model         & Mol Stab $\uparrow$  & Valid $\uparrow$     & NFE   \\
    \midrule
    EQGAT-diff        & 93.4$_{\pm \text{0.21}}$  & 94.6$_{\pm \text{0.24}}$ & 500  \\
    \midrule
    \text{SemlaFlow}$_{50}$ w/ OT   & 97.0$_{\pm \text{0.21}}$  & 93.9$_{\pm \text{0.12}}$  & 50   \\
    \cellcolor{SkyBlue!20}\text{CanonFlow}$_{50}$     & \pmerr{\textbf{98.0}}{0.08}  & \pmerr{\textbf{95.4}}{0.09}  & 50   \\
    \midrule
    \text{SemlaFlow}$_{100}$ w/ OT   & 97.3$_{\pm \text{0.08}}$  & 93.9$_{\pm \text{0.19}}$     & 100  \\
    \cellcolor{SkyBlue!20}\text{CanonFlow}$_{100}$    & \pmerr{\textbf{98.4}}{0.02}  & \pmerr{\textbf{95.9}}{0.08}     & 100   \\
    \midrule
    Data           & 100.0  & 100.0  & -- \\
    \bottomrule
  \end{tabular}
  \vspace{-2pt}
\end{table}

\subsection{Additional Results}

In addition to our state-of-the-art main results presented in \Cref{sec:exp}, we provide more comprehensive ablation studies here.

\paragraph{Superior few-step generation with CanonFlow.} \Cref{table:cfg_rank3} presents few-shot generation results of our CanonFlow (trained with OT annealing), which consistently outperforms the SemlaFlow baseline w/o canonicalization, and the advantage is still significant even with only 50 steps. This further support our claim that canonicalization not only improves the upper bound of generative models, but also provides strong guidance signal through the conditions to facilitate few-step generation.

\begin{table}[ht]
  \caption{Canonical SemlaFlow on GEOM-DRUG dataset w/ classifier-free guidance (CFG). Canonicalization is conducted on $S_N$.}
  \label{table:cfg_sn}
  \centering
  \begin{tabular}{lcccc}
    \toprule
    Model      & CFG scale        & Mol Stab $\uparrow$  & Valid $\uparrow$     & NFE   \\
    \midrule
    EQGAT-diff    & --     & 93.4$_{\pm \text{0.21}}$  & 94.6$_{\pm \text{0.24}}$ & 500  \\
    \midrule
    \text{SemlaFlow}$_{50}$   & --  & 97.0$_{\pm \text{0.21}}$  & 93.9$_{\pm \text{0.12}}$  & 50   \\
    \cellcolor{SkyBlue!20}\text{Canon. SemlaFlow}$_{50}$   & 1.0  & \pmerr{\textbf{97.6}}{0.13}  & \pmerr{94.6}{0.32}  & 50   \\
    \cellcolor{SkyBlue!20}\text{Canon. SemlaFlow}$_{50}$   & 1.5  & \pmerr{97.4}{0.07}  & \pmerr{\textbf{94.7}}{0.12}  & 50   \\
    \cellcolor{SkyBlue!20}\text{Canon. SemlaFlow}$_{50}$   & 2.0  & \pmerr{96.7}{0.04}  & \pmerr{94.5}{0.03}  & 50   \\
    \midrule
    \text{SemlaFlow}$_{100}$  & --  & 97.3$_{\pm \text{0.08}}$  & 93.9$_{\pm \text{0.19}}$     & 100  \\
    \cellcolor{SkyBlue!20}\text{Canon. SemlaFlow}$_{100}$  & 1.0   & \pmerr{\textbf{98.1}}{0.03}  & \pmerr{95.0}{0.20}     & 100   \\
    \cellcolor{SkyBlue!20}\text{Canon. SemlaFlow}$_{100}$  & 1.5   & \pmerr{97.8}{0.03}  & \pmerr{95.0}{0.09}     & 100   \\
    \cellcolor{SkyBlue!20}\text{Canon. SemlaFlow}$_{100}$  & 2.0   & \pmerr{97.5}{0.08}  & \pmerr{\textbf{95.1}}{0.07}     & 100   \\
    \midrule
    Data      & --        & 100.0  & 100.0  & -- \\
    \bottomrule
  \end{tabular}
  \vspace{-2pt}
\end{table}

\begin{table}[ht]
  \caption{Canonical SemlaFlow on GEOM-DRUG dataset w/ classifier-free guidance (CFG). Canonicalization is conducted on $S_N\times SO(3)$.}
  \label{table:cfg_so3}
  \centering
  \begin{tabular}{lcccc}
    \toprule
    Model      & CFG scale        & Mol Stab $\uparrow$  & Valid $\uparrow$     & NFE   \\
    \midrule
    EQGAT-diff    & --     & 93.4$_{\pm \text{0.21}}$  & 94.6$_{\pm \text{0.24}}$ & 500  \\
    \midrule
    \text{SemlaFlow}$_{100}$  & --  & 97.3$_{\pm \text{0.08}}$  & 93.9$_{\pm \text{0.19}}$     & 100  \\
    \cellcolor{SkyBlue!20}\text{Canon. SemlaFlow}$_{100}$  & 1.0   & \pmerr{\textbf{97.9}}{0.09}  & \pmerr{93.8}{0.30}     & 100   \\
    \cellcolor{SkyBlue!20}\text{Canon. SemlaFlow}$_{100}$  & 1.5   & \pmerr{97.7}{0.06}  & \pmerr{94.2}{0.23}     & 100   \\
    \cellcolor{SkyBlue!20}\text{Canon. SemlaFlow}$_{100}$  & 2.0   & \pmerr{97.6}{0.09}  & \pmerr{\textbf{94.4}}{0.05}     & 100   \\
    \midrule
    Data      & --        & 100.0  & 100.0  & -- \\
    \bottomrule
  \end{tabular}
  \vspace{-2pt}
\end{table}

\paragraph{Effects of classifier-free guidance.} We further ablate the sampling quality with CFG. To comprehensively evaluate the effects of canonicalized symmetry groups, we report the performance of canonical SemlaFlow on both $S_N$ (trained w/ OT annealing) in~\Cref{table:cfg_sn} and $S_N\times SO(3)$ (trained w/o OT) in~\Cref{table:cfg_so3}. Within an appropriate range, a larger CFG scale tends to improve the validity metric, at the cost of slightly decreasing the molecule stability. Interestingly, the improvement is more significant on the $S_N\times SO(3)$ canonicalization, since the prediction relies more on the canonical information, which benefits from extrapolating the conditional shortcut and unconditional baseline.

\paragraph{Ablation on aligned prior and OT.} In \Cref{table:prior} we comprehensively study the effects of aligned prior and OT when applied to both equivariant and non-equivariant backbones. We incorporate the $S_N$ symmetry group, and evaluate various sampling steps. Notably, without the aligned prior, an equivariant SemlaFlow (w/o PE) with isotropic priors would be identical to the baseline; in other words, the aligned prior is also a contributing component of our canonical generative models, and even equivariant models could benefit from the better prior-data coupling without the help of PE.

\begin{table}[ht]
  \caption{Canonical SemlaFlow with aligned prior on GEOM-DRUG. Canonicalization is conducted on  $S_N$, and we study both equivariant and non-equivariant backbones w/ or w/o OT. Note that the prior is computed consistently with different group $G$}
  \label{table:prior}
  \centering
  \begin{tabular}{lcccccc}
    \toprule
    Group & Model   & PE  & OT   & Mol Stab $\uparrow$  & Valid $\uparrow$     & NFE   \\
    \midrule
    & EQGAT-diff & - & -       & 93.4$_{\pm \text{0.21}}$  & 94.6$_{\pm \text{0.24}}$ & 500  \\
    \midrule
    & \text{SemlaFlow}$_{20}$  & - & equivariant   & 95.3$_{\pm \text{0.14}}$  & 93.0$_{\pm \text{0.10}}$ & 20   \\
    & \text{SemlaFlow}$_{50}$  & - & equivariant  & 97.0$_{\pm \text{0.21}}$  & 93.9$_{\pm \text{0.12}}$  & 50   \\
    & \text{SemlaFlow}$_{100}$  & - & equivariant  & 97.3$_{\pm \text{0.08}}$  & 93.9$_{\pm \text{0.19}}$     & 100  \\
    \midrule
    & \cellcolor{SkyBlue!20}\text{Canon. SemlaFlow}$_{20} $  & - & equivariant   & \pmerr{95.1}{0.09}  & \pmerr{93.5}{0.17}  & 20   \\
    $\quad S_N$  & \cellcolor{SkyBlue!20}\text{Canon. SemlaFlow}$_{50}$ & - & equivariant    & \pmerr{97.3}{0.08}  & \pmerr{94.4}{0.02}  & 50   \\ 
    & \cellcolor{SkyBlue!20}\text{Canon. SemlaFlow}$_{100}$ & - & equivariant   & \pmerr{97.5}{0.10}  & \pmerr{94.2}{0.11}     & 100   \\
    \midrule

    & \cellcolor{SkyBlue!20}\text{Canon. SemlaFlow}$_{20} $  & True & equivariant   & \pmerr{92.6}{0.18}  & \pmerr{91.7}{0.20}  & 20   \\
    $\quad S_N + PE$  & \cellcolor{SkyBlue!20}\text{Canon. SemlaFlow}$_{50}$ & True & equivariant    & \pmerr{96.3}{0.03}  & \pmerr{93.9}{0.07}  & 50   \\ 
    & \cellcolor{SkyBlue!20}\text{Canon. SemlaFlow}$_{100}$  & True & equivariant  & \pmerr{97.5}{0.10}  & \pmerr{94.2}{0.20}     & 100   \\
    \midrule
    & Data           & 100.0  & 100.0  & -- \\
    \bottomrule
  \end{tabular}
  \vspace{-2pt}
\end{table}

\clearpage
\newpage
\section{Example from Canonical Diffusion}

This section presents samples from Canonicalized SemlaFlow trained on GEOM-Drugs. The samples were generated randomly but we have rotated them where necessary to aid visualization.

\vspace{1cm}
\begin{figure}[htbp]
  \centering
  \includegraphics[width=\linewidth]{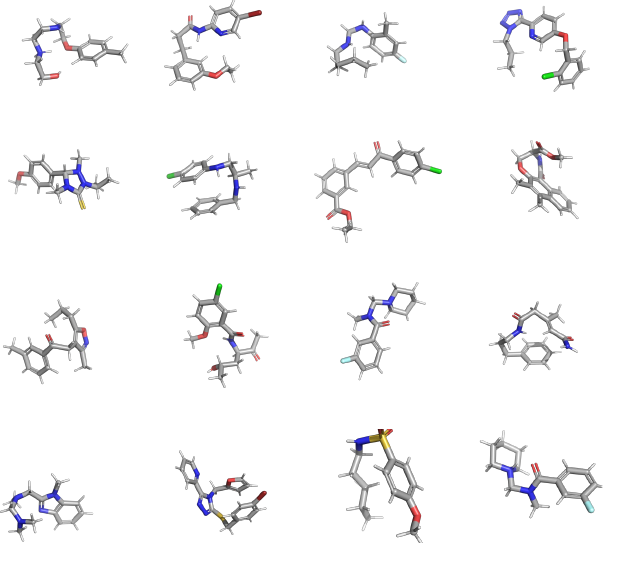}
  \caption{Random samples from Canonicalized Semlaflow model trained on GEOM Drugs. These samples were generated using 100 ODE integration steps.}
  \label{fig:example}
\end{figure}


\end{document}